\documentclass[10pt]{article}

\usepackage[ansinew]{inputenc}
\usepackage{ upgreek }
\usepackage{graphicx} 
\usepackage{textcomp} 
\usepackage{amsfonts} 
\usepackage{amsmath}
\usepackage{amssymb}
\usepackage{yfonts}
\usepackage{mathtools}
\usepackage{mathrsfs}
\usepackage{latexsym}
\usepackage{array}
\usepackage{float}
\usepackage{amsthm}  % error
\usepackage{calc}
\usepackage{perpage}
\usepackage{textcomp}
\usepackage{verbatim} 
\usepackage{hieroglf}
\usepackage{multirow} % can combine rows and columns of tables
\usepackage{rotating} % Can rotate text, can be combined with multirow
\usepackage{dsfont} %for nice indicator

\usepackage{url}
\usepackage{color}
\usepackage{tikz}
\usetikzlibrary{shapes,arrows,positioning,calc}
\usepackage{float}
\usepackage{soul}
\usepackage{enumitem}

\usepackage{booktabs}

\usepackage{algorithm,algorithmic}

\usepackage[normalem]{ulem}

\newcommand{\normsup}[1]{\ensuremath{\| #1 \|_{\infty}}}

\newcommand{\normtwo}[1]{\ensuremath{\| #1 \|_{2}}}
\newcommand{\normq}[1]{\ensuremath{ \| #1 \|_{q}}}

\newcommand{\norm}[1]{\ensuremath{\| #1 \|}}

\newcommand{\normqs}[1]{\ensuremath{ \| #1 \|_{q^\star}}}
\newcommand{\normps}[1]{\ensuremath{ |\!| #1 | \! |_{p^\star}}}
\newcommand{\normp}[1]{\ensuremath{ |\!| #1 | \! |_{p}}}

\newcommand{\tri}[1]{{\left\vert\kern-0.25ex\left\vert\kern-0.25ex\left\vert #1 
    \right\vert\kern-0.25ex\right\vert\kern-0.25ex\right\vert}}

\newtheorem{assumption}{Assumption}

\newcommand{\RN}[1]{%
  \textup{\uppercase\expandafter{\romannumeral#1}}%
}

\newcommand{\1}{{\rm 1}\mskip -4,5mu{\rm l} }

\newcommand{\argmin}{\mathop{\mathrm{arg\,min}}}

\newcommand{\op}[1]{\operatorname{#1}}  % Lettre droite!!!!

\def\B{\mathbb{B}}

\def\E{\mathbb{E}}

\def\cF{\mathcal{F}}

\def\R{\mathbb{R}}

\def\V{\mathcal{V}}
\def\M{\mathcal{M}}
\def\S{\mathcal{S}}

\newcommand{\rp}{\mathbb{R}^p}
\newcommand{\rd}{\mathbb{R}^d}

\newcommand{\mpr}{\mathbb{P}}

\newcommand{\filtration}{\mathcal{F}}

\newcommand{\grad}{\nabla}

\usepackage{wrapfig,lipsum,caption}

\mathtoolsset{showonlyrefs}

\def\balign#1\ealign{\begin{align}#1\end{align}}
\def\baligns#1\ealigns{\begin{align*}#1\end{align*}}
\def\balignat#1\ealign{\begin{alignat}#1\end{alignat}}
\def\balignats#1\ealigns{\begin{alignat*}#1\end{alignat*}}
\def\bitemize#1\eitemize{\begin{itemize}#1\end{itemize}}
\def\benumerate#1\eenumerate{\begin{enumerate}#1\end{enumerate}}

\newcommand{\eq}[1]{\begin{align}#1\end{align}}
\newcommand{\eqn}[1]{\begin{align*}#1\end{align*}}

\makeatletter
\@addtoreset{equation}{section}
\makeatother

\usepackage{hyperref}
\hypersetup{
 colorlinks,
 linkcolor={red!50!black},
  citecolor={blue!50!black},
  urlcolor={blue!80!black}
}

\newcommand{\kl}{D_{\op{KL}}}
\newcommand{\ham}{\Delta_{\op{H}}}

\DeclarePairedDelimiter{\bbrace}{\lbrace}{\rbrace}% 
\DeclarePairedDelimiter{\bbbrace}{\Big \lbrace}{\Big \rbrace}% 

\newcommand*{\inner}[2]{ \langle #1,  #2 \rangle}
\newcommand*{\vpow}[2]{ #1^{\langle #2 - 1 \rangle}} % Vector power

\def\xp{\vpow{x}{p}}
\DeclarePairedDelimiter{\abs}{\lvert}{\rvert}%
\newcommand*{\rom}[1]{\expandafter\@slowromancap\romannumeral #1@}
\usepackage[left=1in,top=1in,right=1in,bottom=1in,letterpaper]{geometry}

\newtheorem{definition}{Definition}
\newtheorem{proposition}{Proposition}
\newtheorem{remark}{Remark}
\newtheorem{theorem}{Theorem}
\newtheorem{lemma}{Lemma}
\newtheorem{corollary}{Corollary}

\mathtoolsset{showonlyrefs}

\title{ Mirror Descent Strikes Again: Optimal Stochastic Convex Optimization under Infinite Noise Variance}

\author{Nuri Mert Vural\thanks{Department of Computer Science at University of Toronto, and Vector Institute. \texttt{vural@cs.toronto.edu}.} 
\ \ \ \ \ \  Lu Yu\thanks{Department of Statistical Sciences at University of Toronto, and Vector Institute. \texttt{stat.yu@mail.utoronto.ca}.  } 
\ \ \ \ \ \  Krishnakumar Balasubramanian\thanks{Department of Statistics at University of California, Davis. \texttt{kbala@ucdavis.edu}. } \\
Stanislav Volgushev\thanks{Department of Statistical Sciences at University of Toronto. \texttt{stanislav.volgushev@utoronto.ca}. } 
\ \ \ \ \ \ Murat A. Erdogdu\thanks{Department of Computer Science and Department of Statistical Sciences at University of Toronto, and Vector Institute. \texttt{erdogdu@cs.toronto.edu}.}
}

\begin{document}

\maketitle

\begin{abstract}%
We study stochastic convex optimization under infinite noise variance. Specifically, when the stochastic gradient is unbiased and has uniformly bounded $(1+\kappa)$-th moment, for some $\kappa \in (0,1]$, we quantify the convergence rate of the Stochastic Mirror Descent algorithm with a particular class of uniformly convex mirror maps, in terms of the number of iterations, dimensionality and related geometric parameters of the optimization problem. Interestingly this algorithm does not require any explicit gradient clipping or normalization, which have been extensively used in several recent empirical and theoretical works. We complement our convergence results with information-theoretic lower bounds showing that no other algorithm using only stochastic first-order oracles can achieve improved rates. Our results have several interesting consequences for devising online/streaming stochastic approximation algorithms for problems arising in robust statistics and machine learning. 
\end{abstract}

\section{Introduction}

For a bounded convex set $\S \subset \mathbb{R}^d$,
and a convex objective function $f: \mathcal{S} \to \mathbb{R}$, we consider the optimization problem
\begin{equation}
  \label{eq:raw-min}
  \underset{x\in \S}{\text{minimize}}\, f(x)\,,
\end{equation}
%where $f$ is bounded below by $f(x^\star)>-\infty$, 
in the stochastic first-order oracle model where one has access
to noisy unbiased gradients at every iteration of an algorithm. This problem naturally emerges in many statistical learning tasks,
thus there has been a substantial amount of research dedicated to understanding convergence guarantees as well as information-theoretic lower bounds in the classical setting where the noise has finite variance~\cite{Bubeck2015,nesterov2018lectures}. However, recent studies have shown empirical and theoretical evidence that stochastic gradients arising from modern learning problems may not have finite variance, in which case the optimal convergence guarantees and computational lower bounds for solving \eqref{eq:raw-min} are not well understood.
 %~\cite{Hodgkinson2021,MG2021,Simsekli2019,zhang2019adaptive,wang2021convergence},

Indeed, heavy-tailed  behavior is ubiquitous in statistical learning. Such behavior may either arise due to the underlying statistical model~\cite{Simsekli2019,zhang2019adaptive,wang2021convergence, gurbuzbalaban2021fractional}
or through the stochastic iterative training process~\cite{Hodgkinson2021,MG2021,camuto2021asymmetric}.
In the regime where stochastic gradients have infinite variance,
while the vanilla stochastic gradient descent (SGD) algorithm converges under strong convexity-type assumptions~\cite{wang2021convergence}, more robust methods like gradient-clipped SGD (used, for example, by~\cite{zhang2019adaptive} for attention models) turn out to have optimal rates under strong convexity when the dimension is treated as a constant. However, it is not clear if gradient-clipped SGD would exhibit similar optimality guarantees in the case of convex problems or when the dimension is not treated as a constant.

In this regard, it is highly desirable to obtain a rigorous understanding of the oracle complexity of stochastic convex optimization in the infinite noise variance setting. Such an understanding boils down to two fundamental questions:
\begin{itemize}[leftmargin=0.4in]
\setlength\itemsep{0.05em}
\item [] An \emph{information-theoretic} question: What is the best achievable lower bound in convex optimization in the stochastic first-order oracle model under infinite noise variance?
\item[] An \emph{algorithmic complexity} question: Is there an optimal optimization algorithm that achieves this information-theoretic lower bound, under the same stochastic first-order oracle model?
\end{itemize}
We provide concrete answers to both of these questions, where the optimal algorithm is, yet again,  stochastic mirror descent (SMD).

Mirror descent is a first-order method which generalizes the standard gradient descent to the non-Euclidean setting by relying on a mirror map that captures the underlying geometric structure of the problem~\cite{Nem1983}. Although originally developed for deterministic frameworks, SMD is
known to achieve the information-theoretic lower bound in the classical stochastic first-order oracle model where the noise has finite variance~\cite{agarwal2012information}. This is remarkable as by simply choosing the appropriate mirror map, one can design algorithms that are optimal in their respective oracle models.
This property of mirror descent has been exploited in many works for establishing the algorithm's optimality in classical settings \cite{nemirovski2009robust, sridharan2012learning}, and for demonstrating its universality in the online setting~\cite{duchi2010composite,srebro2011universality}. In this work, we show that the stochastic mirror descent with an appropriate mirror map has an \emph{inherent robustness} to heavy-tailed gradient noise, and  achieves the information-theoretic lower bound for stochastic convex optimization under infinite noise variance. Towards that we make the following contributions.

\begin{itemize}
\setlength\itemsep{0.05em}
\item We establish the first non-asymptotic convergence of stochastic mirror descent algorithm in the heavy-tailed case where the gradient noise has infinite variance. We provide explicit rate estimates for a class of convex optimization problems in Theorem~\ref{th:main} and Corollary~\ref{cor:main} for a variety of mirror maps.
\item We establish lower bounds for the minimax error in Theorem~\ref{thm:cvx_lb}, for constrained convex optimization in the first-order stochastic oracle model under infinite gradient noise variance.
\item Remarkably, for a careful choice of mirror map which depends on the largest defined moment order of the gradient noise, the stochastic mirror descent achieves the minimax lower bound. This result proves the optimality of the mirror descent algorithm in the heavy-tailed stochastic first-order oracle setting.
\end{itemize}
To the best of our knowledge, 
our results provide the first example of a first stochastic gradient algorithm
that is provably optimal under heavy-tailed noise without explicit gradient clipping, or normalizing the magnitude of the stochastic gradients.
Moreover, our setting covers a wide range of (non-strongly) convex functions,
and the minimax lower bounds we establish are explicit (and optimal) in terms of dimension dependence.

%Rest of the paper is organized as follows.

\bigskip

%\noindent\textbf{Related Works.} 
\subsection{Related work}
Earlier works on stochastic approximation with infinite variance largely focus on investigating the asymptotic behavior of stochastic approximation methods.  \cite{Krasulina1969} first establish the almost sure and $L^p$ convergence for the one-dimensional stochastic approximation process without variance. \cite{Anantharam2012} demonstrate the stability and convergence properties of multivariate stochastic approximation algorithms with the heavy-tailed noise.
Recently, the works of~\cite{Simsekli2019v2}, \cite{zhang2019adaptive},
\cite{chen2020understanding},
and \cite{wang2021convergence} investigate the behavior of SGD under infinite noise variance with various types of objectives. \cite{Simsekli2019v2} considers non-convex optimization and analyze the SGD as a discretization of a stochastic differential equation driven by a L\'evy process.
\cite{zhang2019adaptive} and \cite{chen2020understanding} study the convergence of SGD with gradient clipping, and establish the dimension-free optimal bound with strongly convex and non-convex objectives.
\cite{wang2021convergence} provide the convergence rate of SGD with a strongly convex objective function under a state-dependent and heavy-tailed noise; see also~\cite{mirek2011heavy}. High-probability bounds under certain moment assumptions (but not infinite variance) have also recently been established in~\cite{nazin2019algorithms, cutkosky2021high, davis2021low, gorbunov2021near, tsai2021heavy, lou2022beyond}.

% Stochastic mirror descent algorithm is often analyzed as a stochastic method for optimizing
% non-smooth Lipschitz continuous convex functions~\cite{nemirovski2009robust,Bubeck2015,beck2017first}.
There exists a vast literature on mirror descent algorithm in a stochastic optimization setting
with the stochastic gradient having finite variance~\cite{nemirovski2009robust,Bubeck2015,beck2017first}. Another line of  work~\cite{sridharan2010convex,srebro2011universality} establishes the (near) optimal regret rate of the mirror descent with the aid of uniformly convex mirror maps in a deterministic online setting. SMD was analyzed with almost surely bounded stochastic gradient, for composite optimization problems,  in~\cite{duchi2010composite}. Mirror descent algorithm in the non-i.i.d. setting was considered in~\cite{duchi2012ergodic}.  We emphasize here that these works consider the standard finite variance noise setting, and thus the uniformly convex mirror map proposed in there is inadequate to deal with the infinite variance noise that we focus on in this work. Focusing on the finite-sum setup, \cite{d2021stochastic} investigate the convergence of SMD in (relative) smooth optimization under the finite optimal objective difference assumption~\cite{loizou2021stochastic}, which allows for convergence without bounded gradient or variance assumptions and achieves exact convergence under interpolation.

More broadly, robust statistics is a classical topic with too large a literature to summarize completely. We refer the reader to~\cite{huber2004robust} for an overview. The revival of robust statistics in modern mathematical statistics and learning theory communities arguably started with the work of~\cite{catoni2012challenging}. Since then, there has been intense work on robust mean and covariance estimation~\cite{minsker2015geometric, cardot2017online, minsker2018sub, lugosi2019mean, lugosi2019sub, hopkins2020mean}, and robust empirical risk minimization~\cite{hsu2016loss, diakonikolas2019robust, geoffrey2020robust, lecue2020robust, bartl2021monte}. However, such results are mainly statistical in nature, and they are not directly applicable for the stochastic approximation with heavy-tailed gradients.

\noindent\textbf{Outline of the paper.}
The rest of the paper is organized as follows.
In Section~\ref{sec:prelim}, we provide a definition of the stochastic first-order oracle model considered in this work, and a review of the stochastic mirror descent (SMD) algorithm focusing on uniform convexity and smoothness properties. In Section~\ref{sec:main}, 
we establish the convergence of SMD with a particular choice of mirror map, and illustrate the effect of this choice on heavy-tailed noisy gradient updates. We then provide information-theoretic lower bounds in Section~\ref{sec:lb}, proving the optimality of SMD.
We conclude in Section~\ref{sec:discussion} with a discussion and future directions. All proofs are deferred to the Appendix.

\section{Stochastic Mirror Descent: Preliminaries}
\label{sec:prelim}
%\noindent\textbf{Stochastic first-order oracle.}
Consider a setup in which a convex function $f$ is minimized over a convex and bounded set $\S$, using a stochastic optimization method $M$, which produces the iterate $x_t\in\S$ at iteration $t$. We assume that the sequence of iterates $\{x_t\}_{t\geq 0}$ is adapted to the filtration $\bbrace{\cF_t}_{t \geq0}$ and the method $M$ has access to the following stochastic first-order oracle (\texttt{SFO}).
\begin{assumption}[Stochastic First-order Oracle]\label{as:sfo}
For all $t\geq0$, given the current iterate $x_t$, the \texttt{SFO} produces random variables $f_{t+1} \in \mathbb{R}$ and $g_{t+1} \in \mathbb{R}^d$ that are $\cF_{t+1}$-measurable, satisfying the following two properties.
\begin{enumerate}
\item \textbf{\emph{Unbiasedness:}} For every $t\geq0$, we have
$$\E[f_{t+1}\vert \cF_{t}]=f(x_t)\ \text{ and }\ \E[g_{t+1} \vert \cF_{t}] \in \partial f(x_{t}).$$
\item \textbf{\emph{Finite $(1+\kappa)$-th moment:}}~For some $\kappa \in (0, 1],q\in[1,\infty]$, and $\sigma > 0$, we have
$$ \sup_{t\geq0}\E[\normq{g_{t+1}}^{1+\kappa} \vert \cF_{t}] \leq \sigma^{1+\kappa}.$$
\end{enumerate}
\end{assumption}
Here, $\partial f(x) := \{v\in\mathbb{R}^d \ | \ f(y) \geq f(x) + \inner{v }{y-x} \text{ for all }y\in\mathbb{R}^d\}$ denotes the sub-differential set of $f$ at the point $x$ and $\normq{\cdot}$ denotes the $q$-norm. 
We note that the bounded $(1+\kappa)$-th moment assumption with $\kappa=1$ corresponds to the classical finite noise variance setting;
we are mainly interested in the case where $\kappa<1$, when the variance of the stochastic gradient is undefined. Perhaps, the most popular stochastic optimization method $M$ operating under \texttt{SFO} is the (projected) stochastic gradient descent (SGD) in the Euclidean setting, as given by 
\begin{align}\label{eq:sgd}
\tag{SGD}
y_{t +1}  = x_{t} - \eta g_{t+1} \quad \text{ and }\quad
x_{t+1}  = \argmin_{x \in \S} \normtwo{x - y_{t+1}}^2.
\end{align}

\begin{remark} \label{rem:Lipschitz}
Any function $f$ that is compatible with an 
\texttt{SFO} satisfying Assumption~\ref{as:sfo} must be Lipschitz continuous with respect to $q^\star$-norm with Lipschitz constant $L \leq \sigma$. To see this, we note that a convex function is $L$-Lipschitz on $\S$ in $\normqs{\cdot}$ if and only if 
\[
\sup_{x\in\S} \max_{v \in \partial  f(x)} \norm{v}_{q}\le L,
\] 
where $q,q^\star\in[1,\infty]$ satisfy $\frac{1}{q}+\frac{1}{q^\star}=1$. Moreover, for $v_t\in\partial f(x_t)$ elementary calculations imply
\eq{
\label{eq:lip_eqv}
\norm{v_t}_q  
= \,\norm{\E[ g_{t+1} |\cF_t ]}_q
\le \E[\norm{ g_{t+1} }_q|\cF_t ]
\le \big(\E[\norm{g_{t+1}}_q^{1+\kappa}|\cF_t]\big)^{\tfrac{1}{1+\kappa}}
\le \sigma,~~\forall t\ge 1\,.
}
%Hence, in order for the function class~$\mathcal{H}_{cvx}$ to be consistent with the stochastic first-order oracle under consideration, we require $L\le \sigma$.
\end{remark}

% However,  SGD is easily influenced by a single-stochastic gradient,  which could be very large and incorrect in the heavy-tailed setting.  To circumvent this issue,  in the following,  we consider a larger class of algorithms called the \emph{stochastic mirror descent} algorithm.  

%{\color{red} NMV: I think switching SMD from SGD without any motivation/connection makes this part hard to follow. KB: I think it's ok SV: the current transition reads ok to me}

Mirror descent, first introduced by \cite{Nem1983},  refers to a family of algorithms for first-order optimization~\cite{ Beck2003, PLG2006,Bubeck2015}, which was originally developed to exploit the geometry of the problem.
%It which includes gradient descent as a special case and
Compared to the classical gradient descent for which the iterates are updated along the direction of the negative gradient, in mirror descent, the updates are performed in the ``mirrored'' dual space determined by a transformation called the \emph{mirror map}.
The family of mirror descent algorithms extends naturally to the stochastic first-order oracle setup, which is the main focus of this paper.

For a function $\Psi : \rd \to \R$ that is strictly-convex, continuously differentiable with a norm coercive gradient (i.e. $\lim_{\,\normtwo{x} \to \infty} \norm{ \grad \Psi(x)}_2 = \infty$), we denote its Fenchel conjugate and Bregman divergence respectively
$$\Psi^\star(y) \coloneqq \sup_{x \in \rd} \Big\{\inner{y}{x} - \Psi(x) \Big\}\  \text{ and }\  D_\Psi(x,y) \coloneqq \Psi(x) - \Psi(y) - \inner{\grad \Psi(y)}{y-x}.$$ 
The stochastic mirror descent (SMD) updates are defined as 
\begin{align}
\tag{SMD}
 y_{t+1}  = \grad \Psi^\star \big( \grad \Psi(x_{t}) - \eta g_{t+1} \big) \quad \text{ and }\quad
 x_{t+1}  = \argmin_{x \in \S} D_\Psi(x, y_{t+1}).  \label{eq:smd}
\end{align}
The conditions on $\Psi$ imply that the \eqref{eq:smd} update is well-defined, and $\grad \Psi$ is an invertible map that satisfies $(\grad \Psi)^{-1} = \grad \Psi^\star$~\cite{PLG2006}. The map $\nabla \Psi$ is also also referred to as the \emph{mirror map} and makes ~\eqref{eq:smd} adapt to the geometric properties of the optimization problem.
\vspace{0.1in}

\noindent\textbf{The mirror map.} In the \eqref{eq:smd} update, 
the descent is performed in the dual space which is the mirror image of the primal space under the mirror map.
Different choices of the mirror maps turn out to be suitable for different optimization problems, and the \textit{right} mirror map corresponds to understanding the geometry of the problem, the objective function we minimize
as well as the noise model. 
Notable examples include:
\begin{itemize}[leftmargin=.2in]
\setlength\itemsep{0.05em}
\item \textit{Stochastic Gradient Descent:} For the function $\Psi(x) = \frac{1}{2}\, \normtwo{x}^2$, the mirror map $\grad \Psi$ reduces to the identity map, and its Bregman divergence reduces to $D_\Psi(x,y) = \frac{1}{2}\, \normtwo{x - y}^2$. Therefore, the update rule~\eqref{eq:smd} reduces to the well-known~\eqref{eq:sgd} update.

\item \textit{$p$-norms Algorithm:} For $p \in (1,2]$ and the function $\Psi(x) = \frac{1}{2} \,\norm{x}^2_{p}$, the \eqref{eq:smd} update reduces to the so-called $p$-norms algorithm~\cite{Gentile1999}, which is optimal for stochastic convex optimization under finite noise variance~\cite{agarwal2009information}.

\item \textit{Exponentiated Gradient Descent:} For the function\footnote{The domain of  mirror map can also be defined over a smaller set containing the feasible set, see e.g. \cite{Bubeck2015}.}
 $\Psi(x) = \sum_j x_j \log x_j$, the Bregman divergence
becomes the unnormalized relative entropy,  i.e.,  $D_\Psi(x,y) = \sum_j x_j \log \frac{x_j}{y_j} - \sum_j x_j + \sum_j y_j$, and the update rule~\eqref{eq:smd} corresponds to the exponentiated gradient descent, which is widely used in the prediction with expert advice setting~\cite{PLG2006}.
\end{itemize}
The choice of mirror map is beneficial when dealing with the particular noise model of stochastic gradients. In what follows, we will use uniformly convex mirror maps in the infinite noise variance setting. 

\begin{definition}[Uniform convexity]
Consider a differentiable convex function $\psi : \rd \to \R$, an exponent $r \geq 2,$ and a constant $K > 0$. Then, $\psi$ is $(K,r)$-uniformly convex with respect to $p$-norm if for any $x, y \in \rd$,
\begin{equation}
\psi(y) \geq \psi(x) + \inner{\grad \psi(x)}{y-x} + \frac{K}{r} \norm{x-y}^r_p\,. \label{uni:convineq}
\end{equation}
%\begin{enumerate}
%\item (Zeroth-order)  For any  $x, y \in \rd$ and $\lambda \in (0,1)$,
%\begin{equation}
%\varphi( \lambda x + (1 - \lambda) y ) \leq \lambda \varphi(x) + (1 - \lambda) \varphi(y) - \lambda (1 - \lambda) \big( \lambda^{p-1} - (1 - \lambda)^{p-1} \big) \frac{K}{p} \normq{y - x}^p.
%\end{equation}
%\item (First-order) 
%\end{enumerate}
\end{definition}
Uniformly convex functions with $r = 2$ are known as \emph{strongly convex} in $p$-norm, and the case $p = 2$ reduces to the classical notion of strong convexity in the Euclidean setting. 
% Note that  a uniformly smooth function $\psi$ is always finite and differentiable.
\begin{definition}[Uniform smoothness]
A function $\psi : \rd \to \R$ is $(K,r)$-uniformly smooth with respect to $p$-norm if it is differentiable and if  there exist a constant $K > 0$ and an exponent $r \in (1,2]$ such that for any $x, y \in \rd$, we have
\begin{equation}
\psi(y) \leq \psi(x) + \inner{\grad \psi (x)}{y-x} + \frac{K}{r} \norm{x-y}^r_p\,.
\end{equation}
\end{definition}
Similarly,  uniformly smooth functions with $r = 2$ are known as \emph{strongly smooth} and the case $p=2$ reduces to the classical notion of first-order smoothness in the Euclidean setting. 
%Uniform convexity and uniform smoothness are generalizations of strong convexity and strong smoothness to a larger class of exponents.

%
Uniform convexity and uniform smoothness are dual properties by Fenchel conjugacy~\cite{Zal1983, Penot95}, a property that is better known for their strong versions. Given the norm $\norm{\cdot}_p$ with $p\in [1,\infty]$, denote its associated dual norm by~$\norm{\cdot}_{p^\star}$, where $1/p+1/p^\star=1$. We recall the statement below both for completeness and to obtain quantitative statements later in Proposition~\ref{th:uniconv2} for a special class of uniformly convex functions.
\begin{proposition}
\label{th:uniconv}
Consider a differentiable convex function $\psi: \rd \to \R$, an exponent $r \geq 2,$ and a constant $K > 0$. Then,  $\psi$ is $(K,r)$-uniformly convex with respect to $p$-norm if and only if $\psi^\star$ is $\Big(K^{-\tfrac{1}{r-1}}, \frac{r}{r-1}\Big)$-uniformly smooth with respect to $p^\star$-norm.
\end{proposition}

Next, we quantify the uniform convexity and smoothness parameters of functions of the form $\frac{1}{r} \norm{\cdot}^r_p$,  for $p , r \in (1, \infty)$. Gradients of these functions with an appropriate choice of $r$ and $p$ will be used as mirror maps in the \eqref{eq:smd} update, which will ultimately achieve the minimax lower bound in the heavy-tailed stochastic oracle setup.  
It is well-known that $\frac{1}{2} \norm{x}^2_p$ is $(p-1)$-strongly smooth for $p \in [2, \infty)$ with respect to $p$-norm\footnote{Equivalently,  $\frac{1}{2} \norm{x}^2_{p^\star}$  is $\frac{1}{p-1}$-strongly convex for $p^\star \in (1, 2]$ with respect to.  $p^\star$-norm~\cite{Kakade2009}. }, see e.g.~\cite[Ex.~3.2]{Jud2008}. The next proposition extends this result to $p$-norm with an arbitrary exponent.
\begin{proposition}
\label{th:uniconv2}
For $\kappa \in (0,1]$, $p \in [1+\kappa,  \infty)$ and $p^\star$ satisfying $\frac{1}{p}+\frac{1}{p^\star}=1$, we define 
\begin{align}\label{eq:mainmap}
\hspace{-0.4in}K_p \coloneqq 10 \max \Big\{ 1, (p-1)^{\frac{1+\kappa}{2}} \Big\},~~
    \varphi(x) \coloneqq  \frac{1}{1+\kappa}\, \norm{x}^{1+\kappa}_p ~~\text{and}~~ \varphi^\star(y) \coloneqq  \frac{\kappa}{1+\kappa}\, \norm{y}_{p^\star}^{\tfrac{1+\kappa}{\kappa}}.
\end{align}
Then, the following statements hold for the Fenchel conjugate functions $\varphi$ and $\varphi^\star$.
\begin{enumerate}
\item  $\varphi$ is  $(K_p,  1+\kappa)$-uniformly smooth with respect to $p$-norm.
\item $\varphi^\star$ is  $\Big(K_p^{-\tfrac{1}{\kappa}}, \frac{1+\kappa}{\kappa}\Big)$-uniformly convex with respect to $p^\star$-norm.
\end{enumerate}
\end{proposition}
We emphasize that both $\varphi$ and $\varphi^\star$ in~\eqref{eq:mainmap}  depend on the choice of $p$ and consequently $p^\star$. 
We also note that when $\kappa=1$, Proposition~\ref{th:uniconv2} recovers the strong convexity/smoothness parameter of $\frac{1}{2} \norm{x}_p^2$ up to a constant factor~\cite{Jud2008,Kakade2009}.

\section{Convergence of SMD with a Uniformly Convex Potential }
\label{sec:main}

We now present our main convergence result for the SMD algorithm with a uniformly convex potential, under an \texttt{SFO} that satisfies Assumption~\ref{as:sfo}.
\begin{theorem}
\label{th:main}
Let Assumption~\ref{as:sfo} hold for some $q\in[1,\infty]$ and define $q^\star$ through $\frac{1}{q}+\frac{1}{q^\star}=1$.
For a function $\Psi$ which is $(1,  \frac{1+\kappa}{\kappa})$-uniformly convex with respect to $q^\star$-norm, the~\eqref{eq:smd} algorithm with the corresponding mirror map $\nabla \Psi$, initialized at $x_0 = \argmin_{x \in \S} \Psi(x)$ and run with step size
\begin{align*}
\eta = \frac{R_0^{{1}/{\kappa}}}{\sigma} T^{-\tfrac{1}{1+\kappa}},\quad\text{where}\quad R_0^{\tfrac{1+\kappa}{\kappa}} := \frac{1+\kappa}{\kappa}~ \sup_{x \in \S} \big\{\Psi(x) - \Psi(x_0)\big\}\,
\end{align*}
satisfies
\begin{align}
\E \Bigg[f \Bigg( \frac{1}{T} \sum_{t=0}^{T-1} x_{t} \Bigg) - \min_{x\in\S} f(x)\Bigg] \leq  R_0 \,\sigma\, T^{-\tfrac{\kappa}{1+\kappa}}. \label{eq:mainres}
\end{align}
\end{theorem}
To our knowledge, the above result is the first convergence result for stochastic mirror descent under a noise model that allows infinite noise variance. In contrast to other gradient-based methods in the literature dealing with heavy-tailed noise~\cite{zhang2019adaptive, Gorbunov2020},  stochastic mirror descent does not require (explicit) gradient clipping or gradient normalization to guarantee convergence. In  Section \ref{sec:illustration},  we  will present an instance to illustrate the intuition behind this result.

The error at initialization $R_0$
in the above bound \eqref{eq:mainres} introduces dimension dependency to the rate. To make this explicit,  in the next corollary, we fix the domain as  $\S = \B_\infty(R)$,  where $\B_\infty(R)$ is the $\normsup{\cdot}$-ball with radius $R$, centered at the origin,  and use a specific uniformly convex function as the mirror map.

\begin{corollary}
\label{cor:main}
Let $U_{p}(x) \coloneqq K_{p}^{\frac{1}{\kappa}} \varphi^\star(x)$, where $K_p$ and $\varphi^\star$ are defined in~\eqref{eq:mainmap}, and $\S = \B_\infty(R)$. Under the conditions of Theorem \ref{th:main},
the following statements hold.
\begin{enumerate}[label=\roman*),font=\itshape]
\item For $q \in [1, 1+ \kappa]$, \eqref{eq:smd} with $\Psi\coloneqq U_{p}$\, for $p=1+\kappa$
satisfies
\begin{equation}
\E \Bigg[ f \Bigg( \frac{1}{T} \sum_{t = 0}^{T-1} x_t \Bigg) - \min_{x\in\S} f(x) \Bigg] \leq 10 \,  R \,\sigma\,  \Big( \frac{d}{T} \Big)^{\tfrac{\kappa}{1+\kappa}}.
\end{equation}
\item  For $q \in (1+ \kappa, \infty)$, \eqref{eq:smd} with $\Psi\coloneqq U_{p}$\, for $p=q$ satisfies
\begin{equation}
\E \Bigg[ f \Bigg( \frac{1}{T} \sum_{t = 0}^{T-1} x_t  \Bigg)
- \min_{x\in\S} f(x)\Bigg] \leq 10  \max \big\{1, \sqrt{q-1}\big\} R\, \sigma \, \frac{d^{1 - \frac{1}{q}}}{T^{\tfrac{\kappa}{1+\kappa}}}.
\end{equation}
\item For $q \in (\log d, \infty]$, \eqref{eq:smd} with $\Psi\coloneqq U_{p}$\, for $p=1+\log d$ satisfies
\begin{equation}
\E \Bigg[ f \Bigg( \frac{1}{T} \sum_{t = 0}^{T-1} x_t  \Bigg)
- \min_{x\in\S} f(x)\Bigg] \leq 10  \, R\, \sigma \sqrt{\log d} ~ \frac{d^{1 - \tfrac{1}{q}}}{T^{\tfrac{\kappa}{1+\kappa}}}.
\end{equation}
\end{enumerate}
\end{corollary}

\begin{remark}
Note that part (ii) of Corollary~\ref{cor:main} also covers the case $q \in (\log d, \infty)$; however, the result in part (iii) provides a better convergence rate in terms of dimension dependence. Part (iii) includes the boundary case $q=\infty, q^\star=1$ at the expense of additional $\sqrt{\log d}$ factor in the rate.
\end{remark}

Corollary~\ref{cor:main} provides explicit rates for SMD in the infinite noise variance case ($\kappa<1$), with an explicit mirror map. It also recovers the known optimal rates in the finite variance case ($\kappa = 1$)~\cite{Nem1983,agarwal2009information,agarwal2012information}, in which case it reduces to the well-known $p$-norms algorithm~\cite{Gentile1999}. 

\subsection{Robustness of SMD under heavy-tailed noise}
\label{sec:illustration}

\begin{wrapfigure}{R}{0.45\textwidth}  
%\vspace{-6.6mm}
\centering 
\includegraphics[width=0.45\textwidth]{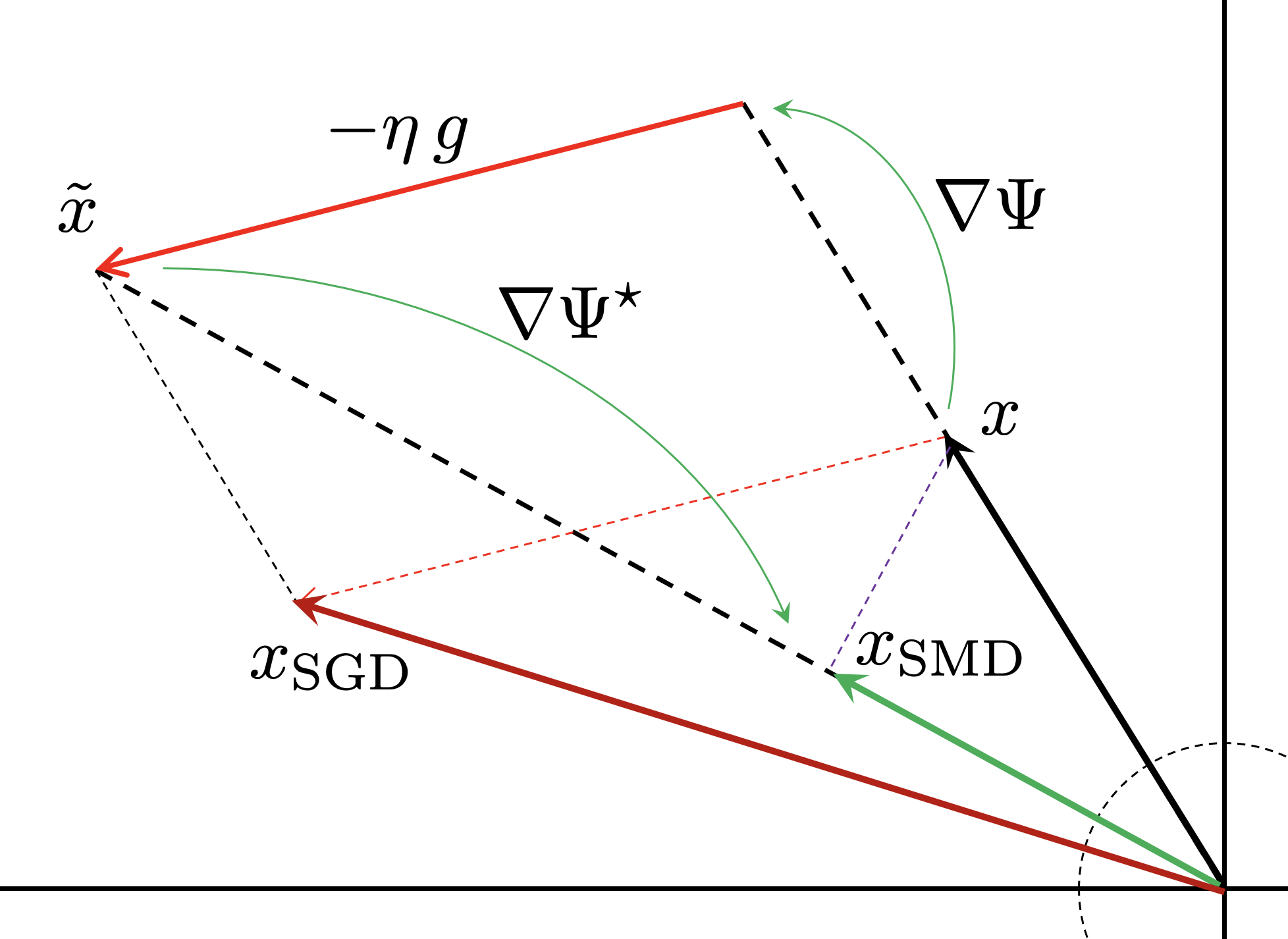}  
%\vspace{-4.mm} 
\captionsetup{width=0.45\textwidth}
\caption{\small Illustration of SMD and SGD updates under heavy-tailed noise.}
\label{fig:clip}    
\vspace{-15mm}
\end{wrapfigure}
We consider a particular instance of \eqref{eq:smd} update 
to provide additional intuition behind the result. 
Let $q = 2$, $\Psi = U_2$, where $U$ is defined in Corollary \ref{cor:main}. Denote the current
iterate by $x$. Based on the noisy gradient $g$ returned by the \texttt{SFO}, the \eqref{eq:smd} update (without projection) is given by

\begin{align}
\!\!\!\!\!\!\!\!
x_{\text{SMD}} &= \grad \Psi^\star \big(   \tilde x \big)\ \text{ for }\ \tilde x  = \grad \Psi(x) - \eta\, g\ \
\label{eq:mirrordyn}
\\
&= \frac{ x \norm{x}_2^{\tfrac{1}{\kappa} - 1} - \frac{\eta}{10^{1/\kappa}} g }{\Big \| x \norm{x}_2^{\tfrac{1}{\kappa} - 1} - \frac{\eta}{10^{1/\kappa}} g \Big\|^{1-\kappa}_2 }. \label{eq:closedformdyn}
\end{align}

For $q=2$, the primal and the dual spaces are both $(L^2(\R^d), \|\cdot\|_2)$, and Figure~\ref{fig:clip} shows the 
updates in the same space for simplicity, and to illustrate the
robustness of \eqref{eq:smd}
in comparison to \eqref{eq:sgd}. In the case where $g$ is large due to heavy-tailed noise, SGD update would be significantly impacted, whereas
SMD first amplifies the magnitude of the iterate $x$, 
then performs the noisy gradient update in the ``dual space'' to get $\tilde x$, and finally contracts the resulting value to $x_{\text{SMD}}$.  
This mechanism of SMD
is illustrated in \eqref{eq:closedformdyn}. 
The descent is performed in the dual space, and the inverse mirror map \emph{shrinks} vectors that are larger in magnitude more when mapping it back to the primal space. This provides an inherent regularization, preventing instabilities due to heavy-tailed noise. 

We formally prove in the next section that SMD remains optimal for the case $\kappa <1$ (i.e., even for the case when the stochastic gradients have infinite noise variance).

\section{Information-theoretic Lower Bounds}
\label{sec:lb}
In this section, we prove that the rates obtained in Theorem~\ref{th:main} and Corollary~\ref{cor:main} are minimax optimal in an information theoretical sense, up to constants and log(dimension) factors. To prove this result, we provide lower bounds on the convergence of any algorithm with access to an \texttt{SFO} satisfying Assumption~\ref{as:sfo}, by extending ideas of \cite{Nem1983} and \cite{agarwal2009information,agarwal2012information} to the infinite noise variance setting. We now give a formal definition of minimax complexity of optimization algorithms in the heavy-tailed setting.

%We extend the approaches by \cite{Nem1983} and \cite{agarwal2009information,agarwal2012information}
%to the infinite noise variance setting, and measure the computational cost based on the heavy tailed stochastic oracle model. We outline the setting as below.

For a convex and compact set $\S$, consider the function class~$\mathcal{H}_{cvx}$ consisting of all convex functions~$f:\S\to\R,$ that are $L$-Lipschitz with respect to ${q^\star}$-norm. That is, 
\begin{equation}
\label{eq:cvxclass}
\mathcal{H}_{cvx}(\mathcal{S},L,q^\star):=\{f:\S\to\R: f \text{ is convex and $L$-Lipschitz with respect to ${q^\star}$ norm}\}\,.
\end{equation}
Recall from Remark~\ref{rem:Lipschitz} that any oracle satisfying Assumption~\ref{as:sfo} must operate on an objective function $f \in \mathcal{H}_{cvx}(\mathcal{S},L,q^\star)$ with $L\leq \sigma$. Thus in our minimax bounds in the sequel, we will only consider such convex and Lipschitz functions.

 Recall that a \texttt{SFO}, which we denote as $\phi$, takes the current iterate $x_t$ and returns the noisy unbiased pair $(f_t, g_t)$ satisfying Assumption~\ref{as:sfo}. We denote by $\Phi(\kappa,q,\sigma)$, the class of all such \texttt{SFO}s with parameters $(\kappa,q,\sigma)$ appearing in Assumption~\ref{as:sfo}. Given an oracle $\phi\in\Phi(\kappa,q,\sigma)$,  let $\mathcal{M}_T$ represent the class of all optimization methods that query the oracle $\phi$ exactly $T$ times and return $\bar x_T\in\S$ as an estimate of the optimum $\argmin_{x\in\S}f(x)$ based on those queries. For any method $M_T\in\mathcal{M}_T$, consider the error in optimizing $f$ after $T$ iterations, 
\begin{equation}
\label{eq:fcn_er}
\epsilon({M}_T, f, \S,\phi) \coloneqq f(\bar x_T)-\min_{x\in\S}f(x)\,.
\end{equation}
Here, $\bar x_T$ should be seen as the output of the method $M_T$ after $T$ iterations, not necessarily the $T$-th iterate of the optimization method.
For example, Theorem~\ref{th:main} and Corollary~\ref{cor:main} provide upper bounds on the expected value of $\epsilon({M}_T, f, \S,\phi)$ for optimization method~$M_T$ corresponding to specific instances of~\eqref{eq:smd}, and $\bar x_T$ corresponds to the average of \eqref{eq:smd} iterates. 

To provide lower bounds on the best possible performance, uniformly over all functions $f \in \mathcal{H}_{cvx}$, of any optimization method $M_T \in \mathcal{M}_T$, we define the minimax error as
\eq{
\label{eq:minimax}
\epsilon_T^*(\mathcal{H}_{cvx},\S,\phi):= \inf_{M_T\in\mathcal{M}_T}\sup_{f\in\mathcal{H}_{cvx}}\E_\phi[\epsilon({M}_T, f, \S,\phi)]\,.
}
The following theorem characterizes the minimax oracle complexity of optimization over the function class~$\mathcal{H}_{cvx}$,
where the constraint set $\S$ is convex and contains $\B_\infty(R)$, the $\|\cdot\|_\infty$-ball of radius $R$ centered at the origin.

\bigskip
\begin{theorem}\label{thm:cvx_lb}
%Consider optimization over the bounded convex set $\mathcal{S}\supseteq \B_\infty(R)$.
Assume that $\mathcal{S}\supseteq \B_\infty(R)$. We have the following minimax lower bounds
%The minimax oracle complexity over the function class~$\mathcal{H}_{cvx}$ satisfies \todo{I don't like this wording, nor the previous one. Please feel free to change it}
\begin{enumerate}
\item \label{thm:cvx_lb1} For all $q\in[1,1+\kappa]$, we have 
\eqn{
\sup_{\phi\in \Phi(\kappa,q,\sigma)}\epsilon^*_T (\mathcal{H}_{cvx},\S,\phi)  \ge\,   C_1{R}L \Bigg(\frac{d}{T}\Bigg)^{\tfrac{\kappa}{1+\kappa}}\,.
}
\item  \label{thm:cvx_lb2} 
For all $q\in (1+\kappa,\infty]$, we have
\eqn{
\!\!\!\!\!\!\sup_{\phi\in \Phi(\kappa,q,\sigma)}\epsilon^*_T (\mathcal{H}_{cvx},\S,\phi)  \ge \, C_2RL\frac{d^{1-\frac{1}{q}}}{T^{\tfrac{\kappa}{1+\kappa}}}\,.
}
\end{enumerate}
Here, $C_1$ and $C_2$ are universal constants.
\end{theorem}
\begin{remark}
The above minimax lower bounds match the rate estimates in Corollary~\ref{cor:main} up to constants for $q\in[1,\infty)$, and an additional $\sqrt{\log d}$ factor for the case of $q=\infty$, proving the optimality of the stochastic mirror descent in the infinite noise variance setting.
\end{remark}

In contrast to existing lower bounds on the oracle complexity in stochastic convex optimization~\cite{agarwal2009information,agarwal2012information,ramdas2013optimal,iouditski2014primal},
Theorem~\ref{thm:cvx_lb} covers a wider range of stochastic first-order oracles. It extends existing minimax lower bounds
to the heavy-tailed noise setting with $\kappa < 1$. Our results recover the information-theoretic lower bound in the classical finite variance setting~\cite[Theorem 1]{agarwal2009information, agarwal2012information}. For the fixed dimension, those bounds can also be linked to limits of results in \cite{ramdas2013optimal,iouditski2014primal} who establish the optimal rate $\Omega\Big(T^{-\tfrac{\rho}{2(\rho-1)}}\Big)$ for $\rho$-uniformly convex functions under finite noise variance. Letting $\rho\to\infty$, which corresponds to convex functions, yields the convergence rate $\Omega(T^{-1/2})$, which is recovered by our results with $\kappa = 1$. Moreover, our lower bounds provide sharp dimension dependence. This extends the findings in \cite[Theorem 3]{raginsky2009information}
and~\cite[Section~5.3.1]{Nem1983} who proved a rate of the form~$\Omega\Big(T^{-\tfrac{\kappa}{1+\kappa}}\Big)$ for the first-order stochastic convex optimization under the heavy-tailed noise setting while treating the dimension $d$ as a constant.

%Moreover, we note that the result of the lower bounds generalizes the work of~\cite{agarwal2009information,agarwal2012information} which measure the computational cost based on the first-order oracle model of optimization under finite gradient noise variance.

The proof strategy of the above oracle complexity lower bound involves a standard reduction from stochastic optimization to a hypothesis testing problem problem. %\todo{Should it be hypothesis testing problem ?}
% , and the application of the Fano types of  inequality for the estimation problem. 
% More specifically,  we first construct a subclass of functions parameterized by a subset of the vertices of a $d$-dimensional hypercube with finite cardinality. 
% We then construct a stochastic oracle based on Bernoulli random variables, each of which corresponds to the parameters of the constructed function in the previous step.
% Next, we convert the parameter estimation to the stochastic optimization problem by showing that optimizing any function in this subclass to certain tolerance requires identifying the hypercube vertices.
% This can be treated as a multiway hypothesis test based on observations provided by querying the stochastic oracle $T$ rounds.
% Finally, we employ Fano types of inequality to lower bound the probability of misspecification error.
Similar arguments appeared in earlier works~\cite{agarwal2009information,agarwal2012information,raginsky2009information,zhang2019adaptive}. However, those works either considered finite-variance noise or treated the dimension as fixed. Covering heavy-tailed stochastic gradient noise and providing explicit dimension dependence requires a more delicate construction of the function class and the first-order oracles, which may be of independent interest. We refer to Appendix~\ref{sec:applb} for the details.

\section{Discussion}
\label{sec:discussion}
In this work, we showed that stochastic mirror descent, with a particular choice of mirror map, achieves the information-theoretically optimal rates for stochastic convex optimization when the stochastic gradient has finite $(1+\kappa)$-th moment, for $\kappa \in (0,1]$. To do so, on the \emph{algorithmic side} we showed that our choice of mirror-map has an inherent regularization property to prevent instabilities that might occur due to heavy-tailed noise in the stochastic gradient. On the \emph{information-theoretic} side, we provided minimax lower bounds that match the upper bound achieved by the stochastic mirror descent algorithm that we analyze. Our work opens up several interesting directions:
\begin{enumerate}[leftmargin=0.3in]
\setlength\itemsep{0.05em}
    \item The current choice of our step-size parameter requires knowledge of the noise level and $\kappa$ (this is true for all optimization methods that deal with the heavy-tailed noise setting). It is extremely interesting and practically relevant to develop adaptive procedures that achieve optimal rates without knowledge of the problem parameters. 
    \item While our current results are in expectation, establishing results that hold with high-probability in the infinite-noise variance setting would provide an interesting complement to our results.
    \item Developing distributional convergence results for the iterates of~\eqref{eq:smd}, along with related statistical inferential procedures is important for uncertainty quantification.
    \item Finally, examining the performance of~\eqref{eq:smd} in the non-convex setting with infinite-noise variance, is interesting both theoretically and practically.
\end{enumerate}
\bigskip

\section*{Acknowledgements}
KB was supported by a seed grant from Center for Data Science and Artificial Intelligence Research,  UC Davis and NSF Grant DMS-2053918. MAE was supported by NSERC Grant [2019-06167],
CIFAR AI Chairs program, and CIFAR AI Catalyst grant. SV was supported by NSERC Grant [2017-06622] and Connaught New Researcher Award. 

\bigskip
\bibliographystyle{alpha}
\bibliography{bib}
\appendix
\section{Proofs for Section \ref{sec:prelim}}
\label{sec:appprelim}

\subsection{Proof of  Proposition \ref{th:uniconv}}

\begin{proof}[{Proof of  Proposition \ref{th:uniconv}}]
($\Rightarrow$) Since $\psi: \rd \to \R$ is uniformly convex and differentiable,  by Proposition \ref{th:differentiable},  $\psi^\star$ is differentiable.  Moreover, since $\psi$ is continuous and convex,  by \cite[Corollary 23.5.1.]{Rock1970},  $(\grad \psi^\star)^{-1} = (\grad \psi)$.  Let $y_1, y_2 \in \rd$ be two arbitrary vectors, and let  $\grad \psi^\star(y_1) = x_1$.  Then,  we have $\grad \psi(x_1) = y_1$ and by \cite[Theorem 23.5]{Rock1970}
\begin{equation}
\psi(x_1) + \psi^\star(y_1)= \inner{x_1}{y_1}. \label{eq:invyoung}
\end{equation}
We can write that
\begin{align}
&  \psi^\star(y_2) = \sup_{x\in\rd} \bbrace{  \inner{y_2}{x} - \psi(x) }  \\
& \leq \sup_{x\in\rd} \bbbrace{  \!  \inner{y_2}{x} - \Big( \psi(x_1) + \inner{\grad \psi(x_1)}{x - x_1} + \frac{K}{r} \normp{x-x_1}^r  \Big)  \! }~ \textrm{\footnotesize{(by the uniform convexity of $\psi$)}} \\
& =   \sup_{x\in\rd} \bbbrace{ \inner{y_2 - y_1}{ x - x_1} - \frac{K}{r} \normp{x-x_1}^r} - \psi(x_1) +\inner{y_2}{x_1}~\textrm{\footnotesize{(since $\grad \psi(x_1) = y_1$)}} \\
& =   \sup_{x\in\rd} \bbbrace{\inner{y_2 - y_1}{ x - x_1} - \frac{K}{r} \normp{x-x_1}^r}  +  \psi^\star(y_1) +\inner{\grad \psi^\star(y_1)}{y_2 - y_1} \label{eq:ll1} \\ 
& =  \sup_{x\in\rd} \bbbrace{\inner{y_2 - y_1}{K^{- \frac{1}{r-1}} x } - \frac{K}{r} \normp{K^{- \frac{1}{r-1}} x}^r}  + \psi^\star(y_1) + \inner{\grad \psi^\star(y_1)}{y_2 - y_1}  \\
& = \psi^\star(y_1) +\inner{\grad \psi^\star(y_1)}{y_2 - y_1}  + K^{- \frac{1}{r-1}} \frac{r-1}{r} \normps{y_2 - y_1}^{\frac{r}{r-1}}, \textrm{ \footnotesize (by Proposition \ref{prop:conjnorm}) }
\end{align}
where we use $\grad \psi^\star(y_1) = x_1$ and \eqref{eq:invyoung} to obtain \eqref{eq:ll1}. Then, for arbitrary $y_1$ and $y_2 \in \rd$, we have 
\begin{equation}
\psi^\star(y_2) \leq  \psi^\star(y_1) +\inner{\grad \psi^\star(y_1)}{y_2 - y_1}  + K^{- \frac{1}{r-1}} \frac{r-1}{r} \normps{y_2 - y_1}^{\frac{r}{r-1}}. \label{eq:uniconvineq}
\end{equation}
Therefore,  $\psi^\star$ is $(K^{- \frac{1}{r-1}}, \frac{r}{r-1})$-H\"{o}lder smooth with respect to $p^\star$-norm.
\bigskip

\noindent ($\Leftarrow$) Since $\psi$ is continuous and convex,  by~\cite[Theorem 12.2]{Rock1970},  we have
\begin{equation}
\psi(x) =   \sup_{y\in\rd} \bbrace{ \inner{x}{y} - \psi^\star(y)}.
\end{equation}
Let $x_1,  x_2 \in \rd$ be two arbitrary vectors, and let  $\grad \psi(x_1) = y_1$.  Then,  we have $\grad \psi^\star(y_1) = x_1$,  and \eqref{eq:invyoung}.  Let $\overline{K} = K^{- \frac{1}{r-1}}$.  Then,
\eq{
&\psi(x_2)  = \sup_{y \in\rd} \bbrace{\inner{x_2}{y} - \psi^\star(y)} \\
 \geq&  \sup_{y \in\rd} \bbbrace{ \inner{x_2}{y} - \Big( \psi^\star(y_1) + \inner{\grad \psi^\star(y_1)}{y - y_1} +  \overline{K} \frac{r-1}{r}  \normps{y - y_1}^{\frac{r}{r-1}} \Big) } \\
 =&  \sup_{y \in\rd} \bbbrace{ \inner{x_2 - x_1}{y - y_1} -  \overline{K} \frac{r-1}{r}  \normps{y - y_1}^{\frac{r}{r-1}} } - \psi^\star(y_1) +\inner{x_2}{y_1}  ~~\textrm{\footnotesize{(since $\grad \psi^\star(y_1) = x_1$)}} \\
 =&   \sup_{y \in\rd} \bbbrace{ \inner{x_2 - x_1}{ y - y_1}  -  \overline{K} \frac{r-1}{r}  \normps{y - y_1}^{\frac{r}{r-1}}}  +  \psi(x_1) + \inner{\grad \psi(x_1)}{x_2 - x_1} \label{eq:ll2} \\ 
 =&  \sup_{y \in\rd} \bbbrace{ \inner{x_2 - x_1}{ \overline{K}^{-(r-1)} y }  -   \overline{K} \frac{r-1}{r}  \normps{  \overline{K}^{-(r-1)} y}^{\frac{r}{r-1}} }  +  \psi(x_1) + \inner{\grad \psi(x_1)}{x_2 - x_1}  \\
 =& \psi(x_1) +\inner{\grad \psi(x_1)}{x_2 - x_1}  +   \frac{\overline{K}^{-(r-1)}}{r} \normp{x_2 - x_1}^{r} \textrm{ \footnotesize (by Proposition \ref{prop:conjnorm}) } \\
 =&  \psi(x_1) +\inner{\grad \psi(x_1)}{x_2 - x_1}  +  \frac{K}{r} \normp{x_2 - x_1}^{r},
}
where we use  $\grad \psi(x_1) = y_1$ and \eqref{eq:invyoung} in \eqref{eq:ll2}. Therefore,  $\psi$ is $(K, r)$-uniformly convex with respect to $p$-norm.
\end{proof}

\subsection{Proof of  Proposition \ref{th:uniconv2}}

In this part,  we use the following notation.
\begin{itemize}[leftmargin =*]
\item For $x = (x_1, \cdots, x_d)^T \in \rd$ and $p >1$,  we let 
\begin{equation}
\vpow{x}{p} \coloneqq ( \text{sgn}(x_1) \abs{x_1}^{p-1}, \cdots,   \text{sgn}(x_d) \abs{x_d}^{p-1} )^T,
\end{equation}
where for $t \in \R$,
\eq{
\text{sgn}(t)\coloneqq \begin{cases}
1, & \textrm{if } t > 0 \\
0, & \textrm{if } t = 0 \\
-1, & \textrm{if } t < 0.
\end{cases}
}

\item We note that $\normp{x}^r$,  $x \in \rd$, is  continuously differentiable for all $p, r > 1$, with a gradient of 
\begin{equation}
\grad \normp{x}^r = \begin{cases}
r  \normp{x}^{r - p} \vpow{x}{p},    &\mbox{if } x \neq 0 \\
0,                   &\mbox{if } x = 0. 
\end{cases} \label{eq:gradphi}
\end{equation}

In the following, for the sake of convenience, we use an abuse of notation, $\normp{0}^{r - p} \vpow{0}{p} \coloneqq 0$, for any $p,  r > 1$.
 %\label{eq:abuseofnot}
\end{itemize}

\subsubsection{Auxiliary Results}

We start with proving some auxiliary results. Let $x \in \rd - \bbrace{0}$ and $y \in \rd$ be two arbitrary vectors.  We let $h \in \rd$ be
\begin{equation}
h \coloneqq \frac{\inner{\xp}{y}}{\normp{x}^p} x,  \textrm{ where } p \in [1+\kappa, \infty).
\end{equation}

\begin{proposition}
\label{prop:innerh}
For $p \in [1 + \kappa, \infty)$, we have 
\begin{enumerate}[label=\roman*),font=\itshape]
\item $\inner{\xp}{y} = \inner{\xp}{h}$
\item $\normp{h} \leq \normp{y}$.
\end{enumerate}
\end{proposition}

\begin{proof}
\begin{enumerate}[label=(\roman*),font=\itshape]
\item Note that
\begin{align}
\inner{\xp}{x} = \sum_{i = 1}^d \abs{x_i}^{p-1} \text{sgn}(x_i) x_i &=  \sum_{i = 1}^d \abs{x_i}^{p-1} \text{sgn}(x_i) \abs{x_i}  \text{sgn}(x_i) \\
& =  \sum_{i = 1}^d \abs{x_i}^{p} = \normp{x}^p.
\end{align}
Hence, 
\begin{align}
\inner{\xp}{h} & = \frac{\inner{\xp}{y}}{\normp{x}^p}  \inner{\xp}{x} =  \inner{\xp}{y}.
\end{align}
\item Note that $\normp{h} =   \frac{\abs{\inner{\xp}{y}}}{\normp{x}^{p-1}}$.  By using H\"{o}lder's inequality,  we can write that
\begin{align}
\abs{\inner{\xp}{y}}  \leq \Big( \sum_{i = 1}^d  ( \abs{x_i}^{p-1} )^{\frac{p}{p-1}} \Big)^{\frac{p-1}{p}} \Big( \sum_{i = 1}^d  \abs{y_i}^p \Big)^{\frac{1}{p}} = \normp{x}^{p-1} \normp{y}.
\end{align}
Hence,
\begin{equation}
\normp{h} =   \frac{\abs{\inner{\xp}{y}}}{\normp{x}^{p1}} \leq \frac{\normp{x}^{p-1} \normp{y}}{\normp{x}^{p-1}} = \normp{y}.
\end{equation}
\end{enumerate}
\end{proof}

\begin{proposition}
\label{prop:orthogonal}
Let $\widetilde \varphi(x) = \frac{1}{r} \normp{x}^r$, where $p,  r \in [ 1 +\kappa, \infty)$. Then, we have
\begin{equation}
\inner{\grad \widetilde \varphi(x+h)}{y-h} = 0.
\end{equation}
\end{proposition}

\begin{proof}
Note that if $x + h = 0$,  the statement is trivially correct. Therefore, without loss of generality, we assume that $x+h \neq 0$. 

By \eqref{eq:gradphi}, we have 
\begin{equation}
\inner{\grad \widetilde \varphi(x+h)}{y-h} =  \normp{x+h}^{r-p} \inner{\vpow{(x+h)}{p}}{y-h}.
\end{equation}
Note that 
\begin{align}
\vpow{(x+h)}{p} &= \vpow{  \bigg(  \Big( \underbrace{ 1 + \frac{\inner{\xp}{y}}{\normp{x}^p} }_{\coloneqq a} \Big)  x \bigg) }{p}  \\
&= \vpow{(ax)}{p} \\
&= ( \abs{ax}^{p-1} \text{sgn}(ax_1), \cdots, \abs{ax_d}^{p-1} \text{sgn}(ax_d) )^T \\
&= \abs{a}^{p-1} \text{sgn}(a)  ( \text{sgn}(x_1) \abs{x_1}^{q-1}, \cdots,   \text{sgn}(x_d) \abs{x_d}^{q-1} )^T \\
&= \abs{a}^{p-1} \text{sgn}(a)   \xp.
\end{align}
Then, 
\begin{align}
\inner{\grad \widetilde \varphi(x+h)}{y-h} & =  \normp{\underbrace{x+h}_{= ax}}^{r-p} \inner{ \underbrace{\vpow{(x+h)}{p}}_{ = \abs{a}^{p-1} \text{sgn}(a)   \xp}}{y-h} \\
& =  \abs{a}^{r - p} \normp{x}^{r - p}  \abs{a}^{p-1} \text{sgn}(a) \inner{\xp}{y-h}  = 0 \textrm{ \footnotesize (by Proposition \ref{prop:innerh})}.
\end{align}
\end{proof}

\begin{proposition}
\label{prop:smoothinh}
For any $p  \in [1+\kappa, \infty)$, we have
\begin{equation}
\normp{x + h}^{1+\kappa} - \normp{x}^{1+\kappa} - (1+\kappa) \normp{x}^{1+\kappa-p} \inner{\xp}{h}  \leq 2 \normp{h}^{1+\kappa}.
\end{equation}
\end{proposition}

\begin{proof}
We have
\begin{align}
\normp{x+h}^{1+\kappa} &= \Big \vert 1 + \frac{\inner{\xp}{y}}{\normp{x}^p} \Big \vert^{1+\kappa} \normp{x}^{1+\kappa} \\
& \leq \Big( 1 + (1+\kappa) \frac{\inner{\xp}{y}}{\normp{x}^p} + 2 \Big \vert \frac{\inner{\xp}{y}}{\normp{x}^p} \Big \vert^{1+\kappa} \Big) \normp{x}^{1+\kappa} \textrm{  \footnotesize  (by Proposition \ref{prop:1dsmooth})} \\
& =  \normp{x}^{1+\kappa} + (1+\kappa) \normp{x}^{1+\kappa-p}  \inner{\xp}{y} + 2  \Big \vert \frac{\inner{\xp}{y}}{\normp{x}^{p-1}} \Big \vert^{1+\kappa} \\
&=   \normp{x}^{1+\kappa} + (1+\kappa) \normp{x}^{1+\kappa-p}  \inner{\xp}{y} + 2 \normp{h}^{1+\kappa}.
\end{align}
\end{proof}

\begin{proposition}
\label{prop:mdsmoothq}
For any $p \in (1,2]$ and $x, y \in \rd$, we have
\begin{equation}
\normp{x+y}^p - \normp{x}^p - p \inner{\xp}{y} \leq 2 \normp{y}^p.
\end{equation}
\end{proposition}

\begin{proof}
Since $p \in (1,2]$,  we have
\begin{align}
\normp{x + y}^p &= \sum_{i =1}^d \abs{x_i + y_i}^p \\
&\leq \sum_{i =1}^d \abs{x_i}^p + p \abs{x_i}^{p-1} \text{sgn}(x_i) y_i + 2 \abs{y_i}^p  \textrm{ \footnotesize  (by Proposition  \ref{prop:1dsmooth})} \\
&=   \normp{x}^p + p \inner{\xp}{y} +  2 \normp{y}^p.
\end{align}
By rearranging the terms, we can obtain the statement.
\end{proof}

\begin{proposition}
\label{prop:mdsmooth2}
For any $p > 2$ and $x, y \in \rd$, we have 
\begin{equation}
\normp{x+y}^2 - \normp{x}^2 - 2 \normp{x}^{2-p} \inner{\xp}{y} \leq (p-1) \normp{y}^2.
\end{equation}
\end{proposition}

\begin{proof}
Since $p-1 > 1$,  the statement holds when $x = 0$. Therefore, in the following, without loss of generality, we assume $x \neq 0$.

\noindent We prove the statement in two different cases, separately.
\begin{itemize}
\item If $\frac{x}{\norm{x}_2} = \pm  \frac{y}{\norm{y}_2}$,  then $y= tx$ for some $t \in \R$.  In that case,
\eq{
\normp{x+y}^2  = (1+ t)^2  \normp{x}^2 &  = (1 + 2 t+ t^2)  \normp{x}^2 \\
& =  \normp{x}^2 + 2 t  \normp{x}^2 + t^2  \normp{x}^2 \\
& =  \normp{x}^2  + 2  \normp{x}^{2-p}   \inner{\xp}{tx} +  \normp{tx}^2 \\
& = \normp{x}^2  + 2  \normp{x}^{2-p}   \inner{\xp}{y} +  \normp{y}^2 \\
& \leq  \normp{x}^2  + 2  \normp{x}^{2-p}   \inner{\xp}{y} + (p-1)  \normp{y}^2. 
}
\item  If $\frac{x}{\norm{x}_2} \neq \pm  \frac{y}{\norm{y}_2}$, then $x \neq t y$ (i.e.,  $x - ty \neq 0$) for all $t \in \R$.  Then,  $g(t) = \normp{x+ty}^2$ is twice continuously differentiable on $\R$,  and
\eq{
g^{\prime\prime}(t) &= - 2 (p -2) \normp{x+ty}^{2-2p} ( \inner{\vpow{(x+ty)}{p}}{y} )^2 \\
&\quad + 2 (p-1) \normp{x+ty}^{2-p} \Big( \sum_{i=1}^d \abs{x_i + ty_i}^{p-2} y_i^2 \Big) \\
& \leq 2 (p -1) \normp{x+ty}^{2-p} \Big( \sum_{i=1}^d \abs{x_i + ty_i}^{p-2} y_i^2 \Big) \\
& \leq 2 (p -1) \normp{x+ty}^{2-p}  \normp{x+ty}^{p-2} \normp{y}^2  \textrm{ \footnotesize  (by H\"{o}lder's inequality)} \\
& = 2 (p-1) \normp{y}^2.
}
Then,  we have
\eq{
\normp{x+y}^2 - \normp{x}^2 - 2 \normp{x}^{2-p} \inner{\xp}{y}& = g(1) - g(0) - g^{\prime}(0) \\
& =  \int_0^{1}  g^{\prime}(t) - g^{\prime}(0)   \,dt  \\
& =  \int_0^{1}  \int_0^{t}   g^{\prime\prime}(u)    \,du \,dt \\ 
& \leq (p-1) \normp{y}^2.
}
\end{itemize}
\end{proof}

\subsubsection{Proof of  Proposition \ref{th:uniconv2}}

\begin{proof}[{Proof of Proposition \ref{th:uniconv2}}]
\begin{enumerate}[leftmargin=*]
\item  We want to show that for all $x, y \in \rd$,
\begin{equation}
\frac{1}{1+\kappa} \normp{x+y}^{1+\kappa} - \frac{1}{1+\kappa}  \normp{x}^{1+ \kappa} -  \normp{x}^{1+\kappa - p} \inner{\xp}{y} \leq  \frac{K_p}{1+\kappa}  \normp{y}^{1+\kappa}.
\end{equation}
Note that since $K_p > 1$, the statement is correct when $x = 0$. Therefore, in the following,  without loss of generality,  we assume $x \neq 0$. Let
\begin{equation}
h =  \frac{\inner{\xp}{y}}{\normp{x}^p} x,  \textrm{ for } p \in [1+\kappa, \infty).
\end{equation}
We will prove the $p \in [1+\kappa,2]$ and $p \in (2, \infty)$ cases separately.

For $p \in [1+\kappa,2]$, by using Proposition \ref{prop:mdsmoothq},  we can write that
\begin{equation}
\normp{x+y}^p - \normp{x + h}^p - \underbrace{p \inner{\vpow{(x+h)}{p}}{y-h}}_{= \: 0 \textrm{ (by Prop. \ref{prop:orthogonal}})} \leq 2 \normp{y-h}^p.
\end{equation}
Therefore, we have 
\begin{equation}
\normp{x+y}^p  \leq \normp{x + h}^p  + 2 \normp{y-h}^p.
\end{equation}
Then, 
\begin{align}
 \normp{x+y&}^{1+\kappa}   \leq (  \normp{x + h}^p + 2 \normp{y-h}^p)^{\frac{1+\kappa}{p}} \\
 \leq& \normp{x + h}^{1+\kappa} + 2 \normp{y-h}^{1+\kappa} \textrm{ \footnotesize (since $1 + \kappa \leq p$, by Proposition \ref{prop:app1})} \\
 \leq&  \normp{x}^{1+\kappa} + (1+\kappa) \normp{x}^{1+\kappa-p} \inner{\xp}{y}  +  2 \normp{h}^{1+\kappa} + 2  \normp{y-h}^{1+\kappa}  \textrm{  \footnotesize  (by Proposition \ref{prop:smoothinh})} \\
 \leq&  \normp{x}^{1+\kappa} + (1+\kappa) \normp{x}^{1+\kappa-p} \inner{\xp}{y}  +  2 \normp{y}^{1+\kappa} + 2 \: 2^{1+\kappa}  \normp{y}^{1+\kappa} \textrm{  \footnotesize  (by Proposition \ref{prop:innerh})} \\
 \leq&   \normp{x}^{1+\kappa} + (1+\kappa) \normp{x}^{1+\kappa-p} \inner{\xp}{y}  +  10 \normp{y}^{1+\kappa} \textrm{ \footnotesize (since $\kappa \in (0,1]$)}. \label{eq:result1}
\end{align}

For $p \in (2, \infty)$, by using Proposition \ref{prop:mdsmooth2}, we can write that
\begin{equation}
\normp{x+y}^2 - \normp{x + h}^2 - \underbrace{2 \normp{x+h}^{2-p} \inner{\vpow{(x+h)}{p}}{y-h}}_{= \: 0 \textrm{ (by Prop. \ref{prop:orthogonal}})} \leq (p-1) \normp{y-h}^2.
\end{equation}
Therefore, we have 
\begin{equation}
\normp{x+y}^2  \leq \normp{x + h}^2  + (p-1) \normp{y-h}^2.
\end{equation}
Then, 
\begin{align}
&\normp{x +y}^{1+\kappa}   \leq (  \normp{x + h}^2  + (p-1) \normp{y-h}^2)^{\frac{1+\kappa}{2}} \\
 \leq& \normp{x + h}^{1+\kappa} + (p-1)^{\frac{1+\kappa}{2}} \normp{y-h}^{1+\kappa}   \textrm{ \footnotesize (since $1 + \kappa \leq 2$, by Proposition \ref{prop:app1})}  \\
 \leq & \normp{x}^{1+\kappa} + (1+\kappa) \normp{x}^{1+\kappa-p} \inner{\xp}{y}  +  2 \normp{h}^{1+\kappa} + (p-1)^{\frac{1+\kappa}{2}}  \normp{y-h}^{1+\kappa}  \textrm{ \footnotesize (by Proposition \ref{prop:smoothinh})} \\
 \leq & \normp{x}^{1+\kappa} + (1+\kappa) \normp{x}^{1+\kappa-p} \inner{\xp}{y}  +  2 \normp{y}^{1+\kappa} + (p-1)^{\frac{1+\kappa}{2}} 2^{1+\kappa}  \normp{y}^{1+\kappa} \textrm{ \footnotesize (by Proposition \ref{prop:innerh})} \\
 \leq&   \normp{x}^{1+\kappa} + (1+\kappa) \normp{x}^{1+\kappa-p} \inner{\xp}{y}  +  10 (p-1)^{\frac{1+\kappa}{2}} \normp{y}^{1+\kappa}  \textrm{ \footnotesize (since $p > 2$ and $\kappa \in (0,1]$)}. \label{eq:result2}
\end{align}
By multiplying both sides in \eqref{eq:result1} and \eqref{eq:result2} with $\frac{1}{1+\kappa}$,  the statement follows.

\item  Let us fix an arbitrary $p \in [1+\kappa,\infty)$.  Note that by Proposition \ref{prop:conjnorm},  
\begin{align*}
\varphi(x) = \frac{1}{1+ \kappa}\normp{x}^{1 + \kappa} \quad\text{and}\quad \varphi^\star(y) =\frac{\kappa}{1 + \kappa} \normps{y}^{\frac{1 + \kappa}{\kappa}}
\end{align*}
are convex conjugate pairs.  By the previous part,  we know that $\varphi$ is $(K_p, 1 + \kappa)$-H\"{o}lder smooth with respect to $p$-norm.  Then,  by Proposition \ref{th:uniconv},  $\varphi^\star$ is $(K_p^{- \frac{1}{\kappa}}, \frac{1 + \kappa}{\kappa})$-uniformly convex with respect to $p^\star$-norm.
\end{enumerate}

\end{proof}

\section{Proofs for Section \ref{sec:main}}
\label{sec:appmain}

\subsection{Proof of Theorem \ref{th:main}}

We start with an auxiliary result,  given in~\cite[Lemma 4.1]{Bubeck2015}. % For the sake of completeness,  we recall the statement and provide a proof.\todo{why proof}

\begin{proposition}{\cite[Lemma 4.1]{Bubeck2015}}
\label{prop:genproj} Let $\Psi$ be the mirror function defined in Theorem \ref{th:main}.  For $y \in \rd$, let $\hat y = \argmin_{x \in \S} D_\Psi(x, y)$. Then, for any $x \in \S$,
\begin{enumerate}[label=\roman*),font=\itshape]
\item $\inner{ \grad \Psi(\hat y) - \grad  \Psi(y)}{ \hat y - x} \leq 0$
\item $D_\Psi(x, \hat y) + D_\Psi(\hat y, y) \leq D_\Psi(x,y)$.
\end{enumerate}
\end{proposition}

\iffalse
\begin{proof}
\begin{enumerate}[label=(\roman*),font=\itshape, leftmargin=*]
\item Note that $\grad_x D_\Psi(x,y) = \grad \Psi(x) - \grad \Psi(y)$.  By the first-order optimality condition for convex functions (see \cite[Proposition 1.3]{Bubeck2015}),  for any $x \in \S$ and $y \in \rd$, we have 
\begin{equation}
\inner{ \grad \Psi(y^\star) - \grad  \Psi(y)}{ y^\star - x} \leq 0.
\end{equation}
\item Note that
\begin{equation}
D_\Psi(x, y^\star) + D_\Psi(y^\star, y) - D_\Psi(x,y) =  \inner{ \grad \Psi(y^\star) - \grad  \Psi(y)}{ y^\star - x} \leq 0.
\end{equation}
Therefore,
\begin{equation}
D_\Psi(x, y^\star) + D_\Psi(y^\star, y) \leq D_\Psi(x,y).
\end{equation}
\end{enumerate}
\end{proof}
\fi

\begin{proof}[{Proof of Theorem \ref{th:main}}]
For notational convenience, we let $x^\star = \argmin_{x \in \S} f(x)$. We start with two observations:
\begin{itemize}
\item $g_{t+1} = \frac{1}{\eta} ( \grad \Psi(x_t) - \grad \Psi (y_{t+1}))$,
\item $D_\Psi(x^\star, x_t) + D_\Psi(x_t,y_{t+1}) - D_\Psi(x^\star, y_{t+1}) = \inner{\grad	\Psi(x_t) - \grad  \Psi(y_{t+1})}{x_t - x^\star}$.
\end{itemize}
Then, we write
\begin{align}
 \inner{g_{t+1}}{& x_t - x^\star}  = \frac{1}{\eta}  \inner{\grad	\Psi(x_t) - \grad  \Psi(y_{t+1})}{x_t - x^\star} \\
& =  \frac{1}{\eta} ( D_\Psi(x^\star, x_t) + D_\Psi(x_t,y_{t+1}) - D_\Psi(x^\star, y_{t+1}) ) \\
& \leq   \frac{1}{\eta}  ( D_\Psi(x^\star, x_t) + D_\Psi(x_t,y_{t+1}) - D_\Psi(x^\star, x_{t+1}) - D_\Psi(x_{t+1}, y_{t+1}) ) \textrm{ \footnotesize (by Proposition \ref{prop:genproj})} \\
& =  \frac{1}{\eta}  (   D_\Psi(x^\star, x_t)   - D_\Psi(x^\star, x_{t+1})  + D_\Psi(x_t,y_{t+1})  - D_\Psi(x_{t+1}, y_{t+1}) ). \label{eq:md1}
\end{align}

Note that $D_\Psi(x^\star, x_t)   - D_\Psi(x^\star, x_{t+1})$  will lead to a telescoping sum when summing over $t = 1$ to $t = T$.  Therefore, it remains to bound the other term:
\begin{align}
D_\Psi(x_t,y_{t+1})  - D_\Psi(x_{t+1}, y_{t+1}) & = \Psi(x_t) - \Psi(x_{t+1}) - \inner{\grad \Psi(y_{t+1})}{x_t - x_{t+1}} \\
& \leq \inner{\grad \Psi(x_t) - \grad \Psi(y_{t+1}) }{ x_t - x_{t+1}} - \frac{\kappa}{1+\kappa} \normqs{x_t - x_{t+1}}^{\frac{1+\kappa}{\kappa}} \\
& = \eta \inner{g_{t + 1}}{x_t - x_{t+1}} - \frac{\kappa}{1+\kappa} \normqs{x_t - x_{t+1}}^{\frac{1+\kappa}{\kappa}} \label{eq:md125} \\
& \leq \eta \normq{g_{t+1}}  \normqs{x_t - x_{t+1}} - \frac{\kappa}{1+\kappa}  \normqs{x_t - x_{t+1}}^{\frac{1+\kappa}{\kappa}} \label{eq:md15} \\
& \leq \frac{1}{1+\kappa} \eta^{1+\kappa} \normq{g_{t+1}}^{1+\kappa}. \label{eq:md2}
\end{align}
where we use that $\Psi$ is $(1, \frac{1+\kappa}{\kappa})$ uniformly convex  in \eqref{eq:md125},  and  maximize the right-hand side of \eqref{eq:md15} to obtain  \eqref{eq:md2}.

By \eqref{eq:md1} and \eqref{eq:md2}, we have
\begin{align}
&\frac{1}{T} \sum_{t=0}^{T-1} \inner{g_{t+1}}{x_t - x^\star}  \leq \frac{D_\Psi (x^\star, x_0)}{\eta T} + \frac{\eta^\kappa}{1+\kappa} \frac{1}{T} \sum_{t=0}^{T-1} \normq{g_{t+1}}^{1+\kappa} \\
 \leq & \frac{\Psi(x^\star) - \Psi(x_0)}{\eta T} + \frac{\eta^\kappa}{1+\kappa} \frac{1}{T} \sum_{t=0}^{T-1} \normq{g_{t+1}}^{1+\kappa}  \textrm{ \footnotesize (since $x_0 = \argmin_{x\in \S} \Psi(x)$ and $\S$ is convex)} \\
 \leq & \frac{\kappa}{1+\kappa} \frac{R_0^{\frac{1+\kappa}{\kappa}}}{\eta T} + \frac{\eta^\kappa}{1+\kappa}  \frac{1}{T} \sum_{t=0}^{T-1} \normq{g_{t+1}}^{1+\kappa} .
\end{align}
Note that $x_t$ is $\cF_t$-measurable.  Hence,  we can write that
\begin{align}
\frac{1}{T} \E \Bigg[ \sum_{t=0}^{T-1} \inner{ \E[ g_{t+1} \vert \cF_t]}{x_t - x^\star}   \Bigg] & \leq \frac{\kappa}{1+\kappa} \frac{R_0^{\frac{1+\kappa}{\kappa}}}{\eta T} + \frac{\eta^\kappa}{1+\kappa} \frac{1}{T} \E \Big[ \sum_{i=0}^{T-1} \E [ \normq{g_{t+1}}^{1+\kappa} \vert \cF_t]  \Big] \\
& \leq   \frac{\kappa}{1+\kappa} \frac{R_0^{\frac{1+\kappa}{\kappa}}}{\eta T} + \frac{\eta^\kappa}{1+\kappa} \frac{1}{T} \sum_{i=1}^T \sigma^{1+\kappa},
\end{align}
which for $v_t \in \partial f(x_t)$,  leads to
\begin{align}
\frac{1}{T} \E \Bigg[  \sum_{t=0}^{T-1} \inner{ v_t }{x_t - x^\star}  \Bigg]  & \leq \frac{\kappa}{1+\kappa} \frac{R_0^{\frac{1+\kappa}{\kappa}}}{\eta T} + \frac{\eta^\kappa}{1+\kappa} \frac{1}{T}  \sum_{t=0}^{T-1}  \sigma^{1+\kappa}.
\end{align}
Then,
\begin{align}
\frac{\kappa}{1+\kappa} \frac{R_0^{\frac{1+\kappa}{\kappa}}}{\eta T} + \frac{\eta^\kappa}{1+\kappa}  \sigma^{1+\kappa} & \geq \E \Bigg[ \frac{1}{T}  \sum_{t=0}^{T-1}  f(x_t) -  \min_{x \in \S} f(x) \Bigg] \textrm{ \footnotesize  (by the convexity of $f$)} \\
& \geq \E \Bigg[ f \Bigg( \frac{1}{T}  \sum_{t=0}^{T-1} x_t \Bigg) - \min_{x \in \S} f(x) \Bigg]  \textrm{ \footnotesize  (by Jensen's inequality)}.  \label{eq:md3}
\end{align}
By using $\eta =  \frac{R_0^{\frac{1}{\kappa}}}{\sigma} \frac{1}{T^{\frac{1}{1+\kappa}}}$ in \eqref{eq:md3}, we can obtain the statement.
\end{proof}

\bigskip
\subsection{Proof of Corollary \ref{cor:main}}

\begin{proof}[{Proof of Corollary \ref{cor:main}}]
\begin{enumerate}[label=\roman*),font=\itshape, leftmargin =*]
\item For $q \in [1, 1+ \kappa]$, we have  
\begin{equation}
\E[ \norm{g_t}_{1+\kappa}^{1+\kappa}] \leq \E[ \normq{g_t}^{1+\kappa}] \leq L^{1+\kappa}~\textrm{(since $q \leq 1 + \kappa)$}. 
\end{equation}
Moreover,  $x_0 = 0$ and
\begin{equation}
R_0^{\frac{1+\kappa}{\kappa}} = \frac{1+\kappa}{\kappa} \sup_{x \in \B_\infty(R)} (\Psi(x) - \Psi(x_0))  \leq 10^{\frac{1}{\kappa}}  \sup_{x \in \B_\infty(R)}  \norm{x}_{\frac{1+ \kappa}{ \kappa}}^{\frac{1+\kappa}{\kappa}}.
\end{equation}
Then,  
\begin{equation}
R_0 \leq 10^{\frac{1}{1+\kappa}}    \sup_{x \in \B_\infty(R)}   \norm{x}_{\frac{1+ \kappa}{ \kappa}} \leq 10 R  d^{\frac{\kappa}{1+\kappa}}.
\end{equation}
Since $\Psi =U_p$ for $p = 1 + \kappa$ is $(1,\frac{1+\kappa}{\kappa})$-uniformly convex  with respect to $\frac{1+\kappa}{\kappa}$-norm (see Proposition \ref{th:uniconv2}),  by Theorem \ref{th:main}, we have
\begin{equation}
\E \Bigg[ f \Bigg( \frac{1}{T}  \sum_{t=0}^{T-1} x_t \Bigg) - \min_{x \in \S} f(x) \Bigg]   \leq 10 \:  R  \sigma  \Big( \frac{d}{T} \Big)^{\frac{\kappa}{1+\kappa}}.
\end{equation}
\item   For $q \in (1+ \kappa, \infty)$,   we have $x_0 =0$ and 
\begin{equation}
R_0^{\frac{1+\kappa}{\kappa}} = \frac{1+\kappa}{\kappa}  \sup_{x \in \B_\infty(R)}  (\Psi(x) - \Psi(x_0))  \leq \Big( 10 \max \bbrace{1, (q-1)^{\frac{1+\kappa}{2}}} \Big)^{\frac{1}{\kappa}}  \sup_{x \in \B_\infty(R)}   \normqs{x}^{\frac{1+\kappa}{\kappa}}.
\end{equation}
Then,  
\begin{equation}
R_0 \leq 10^{\frac{1}{1+\kappa}}    \max \bbrace{1, \sqrt{q-1}}  \sup_{x \in \B_\infty(R)}  \normqs{x} \leq 10 \max \bbrace{1, \sqrt{q-1}} R  d^{1 - \frac{1}{q}}.
\end{equation}
Since $\Psi = U_p$ for $p = q$ is $(1,\frac{1+\kappa}{\kappa})$-uniformly convex  with respect to $q^\star$-norm (see Proposition \ref{th:uniconv2}),  by Theorem \ref{th:main},  we have
\begin{equation}
\E \Bigg[ f \Bigg( \frac{1}{T}  \sum_{t=0}^{T-1} x_t \Bigg) - \min_{x \in \S} f(x) \Bigg]   \leq 10  \max \bbrace{1, \sqrt{q-1}} R \sigma  \frac{d^{1 - \frac{1}{q}}}{T^{\frac{\kappa}{1+\kappa}}}.
\end{equation}
\item  If $q  \in (\log d, \infty]$,  we have 
\eq{
\E[ \norm{g_t}_{1+ \log d}^{1+\kappa}]   \leq  \E[  ( d^{ \frac{1}{1+  \log d} - \frac{1}{q}} \normq{g_t} )^{1+\kappa}  ] &= d^{ \frac{1+ \kappa}{1+  \log d} - \frac{1 + \kappa}{q}} \E[ \normq{g_t}^{1+\kappa}  ] ~\textrm{\footnotesize(since $q > \log d)$} \\
& \leq   d^{ \frac{1+ \kappa}{1+  \log d} - \frac{1 + \kappa}{q}}  \sigma^{1+\kappa}. 
}
Moreover,  $x_0 = 0$ and
\begin{equation}
R_0^{\frac{1+\kappa}{\kappa}} = \frac{1+\kappa}{\kappa}   \sup_{x \in \B_\infty(R)}  (\Psi(x) - \Psi(x_0))  \leq  \Big( 10 (1 +  \log d-1)^{\frac{1+\kappa}{2}} \Big)^{\frac{1}{\kappa}}  \sup_{x \in \B_\infty(R)}  \norm{x}_{\frac{1 + \log d}{ \log d }}^{\frac{1+\kappa}{\kappa}}.
\end{equation}
Then,  
\begin{equation}
R_0 \leq 10^{\frac{1}{1+\kappa}} \sqrt{\log d }   \sup_{x \in \B_\infty(R)}   \norm{x}_{\frac{1 + \log d}{\log d}} \leq 10 \sqrt{\log d} R  d^{\frac{ \log d}{1 + \log d}}.
\end{equation}
As $\Psi = U_p$ for $p = 1 + \log d$ is $(1, \frac{1+\kappa}{\kappa})$-uniformly convex with respect to $\frac{(1+\log d)}{\log d}$-norm,  by Theorem \ref{th:main}, we have
\begin{equation}
\E \Bigg[ f \Bigg( \frac{1}{T}  \sum_{t=0}^{T-1} x_t \Bigg) - \min_{x \in \S} f(x) \Bigg]  \leq 10  \: R \sigma \sqrt{\log d}  \frac{d^{1 - \frac{1}{q}}}{T^{\frac{\kappa}{1+\kappa}}}. 
\end{equation}

\end{enumerate}
\end{proof}

\section{Auxiliary Results for Sections \ref{sec:prelim} and \ref{sec:main}}

\begin{proposition}
\label{prop:app1}
Let $x,y \geq 0$ and $\kappa \in (0, 1]$.  Then,
\begin{enumerate}[label=(\roman*),font=\itshape]
\item $(x+y)^\kappa \leq x^\kappa + y^\kappa$
\item $x^\kappa + y^\kappa \leq 2^{1-\kappa} (x+y)^\kappa$.
\end{enumerate}
\end{proposition}

\begin{proof}
\begin{enumerate}[label=(\roman*),font=\itshape]
\item Without loss of generality, assume $x \geq y$.  By concavity, we have
\begin{align}
(x + y)^\kappa &\leq x^\kappa + \kappa x^{\kappa - 1} y \\
&\leq x^\kappa + y^\kappa. \textrm{ \footnotesize (Since  $\kappa \in (0, 1]$ and $y \leq x$)}
\end{align}
\item By using $p = 1 / \kappa$ and $p^\star = 1 / (1-\kappa)$ in H\"{o}lder's inequality,  we write
\begin{align}
x^\kappa + y^\kappa \leq (1^{p^\star} + 1^{p^\star})^{1/{p^\star}} (x+y)^{1/p} = 2^{1-\kappa} (x+y)^\kappa.
\end{align}
\end{enumerate}
\end{proof}

\begin{proposition}
\label{prop:1dsmooth}
Let $x, y  \in \R$ and $\kappa \in (0,1]$. Then,
\begin{equation}
\abs{x+y}^{1+\kappa} - \abs{x}^{1+\kappa} - (1+\kappa) \abs{x}^{\kappa} \text{sgn}(x) y \leq 2^{1-\kappa}  \abs{y}^{1+\kappa}.
\end{equation}
\end{proposition}

\begin{proof}
Let $g(x) = \abs{x}^{1+\kappa}$ for $x \in \R$.  Note that $g$ is convex and continuously differentiable,  where  $g^\prime(x) = (1+\kappa) \abs{x}^\kappa \text{sgn}(x)$. Then,
\begin{align}
\abs{x+y}^{1+\kappa} - \abs{x}^{1+\kappa} - (1+\kappa) \abs{x}^{\kappa} & \text{sgn}(x) y = g(x+y) - g(x) - g^\prime(x) y \\
& = \int_x^{x+y}  (  g^\prime(t) -  g^\prime(x) )  \,dt \\
& \leq  \int_x^{x+y}  \abs{  g^\prime(t) -  g^\prime(x) }  \,dt \\
& =  (1+\kappa)   \int_x^{x+y}  \big \vert  \abs{t}^\beta \text{sgn}(t) -    \abs{x}^\beta \text{sgn}(x) \big \vert   \,dt .  \label{eq:intupper}
\end{align}
In the following, we will find an upper-bound for the integrand in \eqref{eq:intupper}.
\begin{itemize}
\item If $\text{sgn}(t) = \text{sgn}(x)$, we have
\begin{align}
 \big \vert  \abs{t}^\kappa \text{sgn}(t) -    \abs{x}^\kappa \text{sgn}(x)   \big \vert &=  \big \vert  \abs{t}^\kappa -    \abs{x}^\kappa  \big \vert \\
& \leq \abs{t-x}^\kappa \textrm{ \footnotesize  (By Proposition \ref{prop:app1})}.
\end{align}
\item If $\text{sgn}(t) \neq \text{sgn}(x)$,  
\begin{align}
 \big \vert \abs{t}^\kappa \text{sgn}(t) -    \abs{x}^\kappa \text{sgn}(x)   \big \vert  &=   \abs{t}^\kappa +    \abs{x}^\kappa \\
& \leq 2^{1-\kappa} ( \abs{t}^\kappa + \abs{x}^\kappa)  \textrm{  \footnotesize  (By Proposition \ref{prop:app1})} \\
& = 2^{1-\kappa} \abs{t-x}^\kappa.
\end{align}
\end{itemize}
Then, we have
\begin{align}
\eqref{eq:intupper} \leq 2^{1-\kappa}  \int_x^{x+y} (1+\kappa) \abs{t-x}^\kappa   \,dt & = 2^{1-\kappa}  \int_0^{y} (1+\kappa) \abs{t}^\kappa   \,dt \\
& = 2^{1-\kappa}  \abs{y}^{\kappa} \\
& \leq 2 \abs{y}^{\kappa}.
\end{align}
\end{proof}

\begin{proposition}
\label{prop:conjnorm}
Let $r >1$,  $p \in [1, \infty]$,
\begin{equation}
 \widetilde \varphi(x) = \frac{1}{r} \normp{x}^{r} \quad \text{and} \quad \widetilde \varphi^\star(y) = \sup_{x\in\rd} \bbrace{ \inner{y}{x} - \widetilde \varphi(x) }.  
\end{equation}
 Then,  for $1/p + 1/p^\star = 1$, we have 
 \begin{align*}
 \widetilde \varphi^\star(y) =  \frac{r-1}{r} \normps{y}^{\frac{r}{r-1}}.
 \end{align*}
\end{proposition}

\begin{proof}
Let us fix an arbitrary $y \in \rd$.  By using H\"{o}lder's inequality,  for any $x \in \rd$, we can write that
\begin{align}
\inner{y}{x}  -  \frac{1}{r} \normp{x}^{r}  &\leq \normps{y} \normp{x} - \frac{1}{r} \normp{x}^{r} \\
&\leq  \frac{r-1}{r} \normps{y}^{\frac{r}{r-1}} \textrm{ \footnotesize (by maximizing the right-hand side)}.
\end{align}
Therefore, we have
\begin{equation}
\widetilde \varphi^\star(y)  \leq  \frac{r-1}{r} \normps{y}^{\frac{r}{r-1}}.
\end{equation}
Moreover, since dual norm can be formulated as a supremum on a compact set,  there exists a $x \in \rd$ such that $ \inner{y}{x}  =  \normps{y} \normp{x} $ and $\normps{y} = \normp{x}^{r-1}$. In this case,
\begin{equation}
\inner{y}{x}  -  \frac{1}{r} \normp{x}^{r} =  \frac{r-1}{r} \normps{y}^{\frac{r}{r-1}}.
\end{equation}
Therefore, we have
\begin{equation}
\widetilde \varphi^\star(y)  \geq  \frac{r-1}{r} \normps{y}^{\frac{r}{r-1}}.
\end{equation}
Consequently,  $\widetilde \varphi^\star(y)  =  \frac{r-1}{r} \normps{y}^{\frac{r}{r-1}}$.
\end{proof}

\begin{proposition}
\label{th:differentiable}
If $\psi: \rd \to \R$ is differentiable and uniformly convex,  $\psi^\star$ is everywhere differentiable.
\end{proposition}

\begin{proof}
Let us say that $\psi$ is $(K,r)$-uniformly convex with respect to some $p$-norm.  First, we will show that $\psi^\star$ is subdifferentiable by proving that it is everywhere finite.   For any $y \in \rd$,  we have
\begin{align}
\psi^\star(y) & = \sup_{x \in \rd} \{ \inner{y}{x} - \psi(x)   \} \\
& \leq  \sup_{x \in \rd} \Big \{ \inner{y}{x} - (\psi(y) + \inner{\grad \psi(y)}{x-y} + \frac{K}{r} \normp{y-x}^r)  \Big \} \\
& =    \sup_{x \in \rd}  \Big \{  \inner{y - \grad \psi(y)}{x - y} -  \frac{K}{r} \normp{y-x}^r)   \Big \}  - \psi(y) + \norm{y}_2^2 \\
& =    \sup_{x \in \rd} \Big \{ \inner{y - \grad \psi(y)}{K^{- \frac{1}{r-1}} x} -  \frac{K}{r} \normp{K^{- \frac{1}{r-1}} x}^r)  \Big \}  - \psi(y) + \norm{y}_2^2 \\
& = K^{- \frac{1}{r-1}} \frac{r-1}{r}  \normps{y - \grad \psi(y) }^{\frac{r}{r-1}} - \psi(y) + \norm{y}_2^2.
\end{align}
Therefore, for any $y \in \rd$,  $\psi^\star(y)$ is finite.  By~\cite[Theorem 23.4]{Rock1970},  $\psi^\star$ is a subdifferentiable convex function.

Next, we prove an intermediate result.  Since $\psi$ is differentiable and uniformly convex,  we have  
\begin{equation}
\psi(y) \geq \psi(x) + \inner{\grad \psi(x)}{y-x} + \frac{K}{r} \normp{y-x}^r,  \quad \forall x, y \in \rd,  \label{eq:ucov1}
\end{equation}
and
\begin{equation}
\psi(x) \geq \psi(y) - \inner{\grad \psi(y)}{y - x} + \frac{K}{r} \normp{y - x}^r  \quad \forall x, y \in \rd. \label{eq:ucov2}
\end{equation}
By summing \eqref{eq:ucov1} and \eqref{eq:ucov2},  we can write that
\begin{equation}
\inner{\grad \psi(x) - \grad \psi (y)}{x-y} \geq \frac{2K}{r}  \normp{y-x}^r,  \quad \forall x, y \in \rd.  \label{eq:diff1}
\end{equation}

We show that $\psi^\star$ is differentiable by using proof by contradiction.  Choose an arbitrary $y_0 \in \rd$. Since $\psi^\star$ is subdifferentiable, we know that $\partial \psi(y_0) \neq \emptyset$.  Let us assume that $x_1, x_2 \in \partial \psi^\star(y_0)$,  and $x_1 \neq x_2$.   Since $\psi$ is continuous and convex,  by \cite[Corollary 23.5.1.]{Rock1970}, 
\begin{equation}
\grad \psi(x_1) = \grad \psi(x_2) = y_0.  \label{eq:diff2}
\end{equation}  
However,  \eqref{eq:diff1} contradicts with \eqref{eq:diff2}.  Since $\psi^\star$ is subdifferentiable,  there must be a unique element in $ \partial \psi^\star(y_0).$ Therefore,  by~\cite[Theorem 25.1]{Rock1970},  $\psi^\star$ is differentiable at $y_0$.  Since $y_0$ was chosen arbitrarily,  $\psi^\star$ is everywhere differentiable.
\end{proof}

\section{Proofs for  Section \ref{sec:lb}}
\label{sec:applb}
%
%\noindent\textbf{Notation}~~~For any $p>0, x\in\rd,$ define
%\eq{
%x^{\brkt{p-1}}:=(|x_1|^{p-1}\op{sgn}(x_1),\dots, |x_d|^{p-1}\op{sgn}(x_d))\,.
%} 

\subsection{Auxiliary lemmas }

To prove Theorem~\ref{thm:cvx_lb}, we need the following lemmas.

\begin{lemma}[KL-divergence between Bernoulli distributions]
\label{lem:KL}
The KL divergence between two Bernoulli distributions $\textsc{Ber}(1- \frac{2+\alpha}{4}p)$ and $\textsc{Ber}(1-\frac{2-\alpha}{4}p)$ is bounded by $p,$ where $\alpha\in\{-1,+1\},p\in\Big(0,\frac{1}{2} \Big).$
\end{lemma}
\begin{proof}[Proof of Lemma~\ref{lem:KL}]
Denote the Bernoulli distributions~$\textsc{Ber}(1- \frac{3}{4}p)$ and $\textsc{Ber}(1-\frac{1}{4}p)$ by $\mpr^+$ and $\mpr^-$, respectively.
By the definition  KL divergence, it holds that
\begin{align*}
\kl(\mpr^+ || \mpr^-) 
=& \Big(1- \frac{3}{4}p \Big)\log \Big(\frac{1- \frac{3}{4}p}{1- \frac{1}{4}p}\Big) + \frac{3}{4}p \log\Big( \frac{ \frac{3}{4}p}{ \frac{1}{4}p}\Big)
\le  \frac{3}{4}p \log 3\le p
\end{align*}
We now prove 
\begin{align*}
\kl( \mpr^- ||\mpr^+) 
=& (1- \frac{1}{4}p)\log \Big(\frac{1- \frac{1}{4}p}{1- \frac{3}{4}p} \Big)+ \frac{1}{4}p \log \Big(\frac{ \frac{1}{4}p}{ \frac{3}{4}p}\Big)
%\le  ( 1- \frac{1}{4}p )p + \frac{1}{4}p\log 1 
\le   p\,.
\end{align*}
Define the function $h(p)=p-\frac{1}{4}p\log(\frac{1}{3})-(1-\frac{1}{4}p)\log(\frac{4-p}{4-3p}).$
Then, we obtain
\eq{
\grad h(p) = 1-\frac{1}{4}\log\Big(\frac{1}{3}\Big)+\frac{1}{4}\log\Big(\frac{4-p}{4-3p}\Big) -\frac{2}{4-3p}\,,
}
and
\eq{
\grad^2 h(p) =
-\frac{16}{(4-p)(4-3p)^2}\,.
}
When $p\in(0,\frac{1}{2}],$ it follows that
\begin{align*}
\grad^2 h(p)  <0,\quad \grad h\big(\frac{1}{2}\big)\ge 0.5 \quad\text{and}\quad h(0)&=0\,,
\end{align*}
which implies $h(p)\ge 0,$ that is
\eq{
p\ge \frac{1}{4}p\log\Big(\frac{1}{3}\Big)+\Big(1-\frac{1}{4}p\Big)\log\Big(\frac{4-p}{4-3p}\Big)\,.
}
\end{proof}

\begin{lemma}[Lower bound~\ref{thm:cvx_lb1} with $d\ge 2$] 
\label{lem:lb_cointoss}
Suppose a vector~$\alpha^*=(\alpha^*_1,\dots,\alpha^*_d)^\top$ is chosen uniformly at random from the set $\mathcal{V},$ where $\mathcal{V}$ is a subset of the hypercube~$\{-1,+1\}^d$ such that $\ham(\alpha,\tilde{\alpha})=\sum_{i=1}^d \1\{\alpha_i\neq\tilde\alpha_i\}\ge \frac{d}{4}$ for any $\alpha,\tilde{\alpha}\in\mathcal{V}.$  
Given the vector~$\alpha^*,$ $\kappa\in(0,1],$ and $\delta\in(0,\frac{1}{8}],$ set the parameter~
$$
\tilde\alpha^*=\Big(1-\frac{2+\alpha^*_1}{4}(4\delta)^{\frac{\kappa+1}{\kappa}},\dots,1-\frac{2+\alpha_d^*}{4}(4\delta)^{\frac{\kappa+1}{\kappa}} \Big)^\top.
$$
Suppose the oracle~$\phi$ tosses a set of $d$ coins with bias~$\tilde\alpha^*$  a total of $T$ times, and the outcome of only one coin chosen uniformly at random is given at each round. 
When $d\ge 2,$ it holds for any holds for any estimator~$\hat{\alpha}\in\mathcal{V}$ that
\eq{
\mpr(\hat\alpha\neq \alpha^*)\ge 1-\frac{(4\delta)^{\frac{\kappa+1}{\kappa}}T+\log 2}{d/8}\,.
}
Here, the probability is taken over the randomness of $\alpha^*$ and $\phi.$
%When $d=1,$  it holds for any estimator~$\hat{\alpha}\in\mathcal{V}$ that
%\eq{
% \mpr_{}(\hat\alpha\neq \alpha^*)\ge 1-\sqrt{\frac{(4\delta)^{\frac{\kappa+1}{\kappa}}T}{2}}\,.
%}
\end{lemma}
\begin{proof}[Proof of Lemma~\ref{lem:lb_cointoss}]
%Denote the Bernoulli distribution for $i$-th coin by $P_{\tilde\alpha_i}.$ 
Let $U_t\in\{1,\dots,d\}$ be the variable indicating the~$U_t$-th coin revealed at time $t,$ and let $X_t\in\{0,1\}$ denote its outcome.
By~\cite[Sec 15.3.2, Lemma 4]{pollard2012festschrift} and~\cite[Theorem 1]{scarlett2019introductory}, if the parameter $\alpha^*$ is uniform on $\mathcal{V},$ it holds for any estimator $\hat\alpha\in\mathcal{V}$ that
\eq{
\mpr(\hat\alpha\neq \alpha^*)\ge 1-\frac{I\big(\{U_t,X_t\}_{t=1}^T;\alpha^*\big)+\log 2}{\log |\mathcal{V}|}\,,
}
where $I(\{U_t,X_t\}_{t=1}^T;\alpha^*)$ denotes the mutual information between the data sequence~$\{U_t,X_t\}_{t=1}^T$ and $\alpha^*$.
By the Varshamov-Gilbert bound~\cite[Lemma 4.7]{massart2007concentration}, 
there exists such a packing set~$\mathcal{V}\subseteq \{-1,+1\}^d$ with $|\mathcal{V}|\ge \exp(\frac{d}{8})$ satisfies
$\ham(\alpha,\tilde{\alpha})=\sum_{i=1}^d \1\{\alpha_i\neq\tilde\alpha_i\}\ge \frac{d}{4}$ for any $\alpha,\tilde{\alpha}\in\mathcal{V}.$  
It suffice to show that $I(\{U_t,X_t\}_{t=1}^T; \alpha^*) \le (4\delta)^{\frac{\kappa+1}{\kappa}} T.$
By the independent and identically distributed the sampling, we have
\eq{
I(\{U_t,X_t\}_{t=1}^T; \alpha^*)=\sum_{t=1}^T I\big( (U_1,X_1); \alpha^*\big)=TI\big( (U_1,X_1); \alpha^*\big)\,.
}
By chain rule of mutual information and the sampling scheme, it holds that 
\eq{
I\big( (U_1,X_1); \alpha^*\big)=I(X_1;\alpha^*|U_1)+I(\alpha^*;U_1)\,.
}
Note that $U_1$ is sampled independent of $\alpha^*,$  this implies $I(\alpha^*;U_1)=0.$
It remains to show that $I(X_1;\alpha^*|U_1)\le (4\delta)^{\frac{\kappa+1}{\kappa}}.$
By definition of the conditional mutual information, and the factorization~$\mpr_{X_1,\alpha^*|U_1}=\mpr_{\alpha^*|U_1}\mpr_{X_1|\alpha^*,U_1},$ it holds that 
\eq{
I(X_1;\alpha^*|U_1)= \E_{U_1}\big[ \kl(\mpr_{X_1|\alpha^*,U_1}||\mpr_{X_1|U_1}) \big]\,.
}
Assume a random vector $\alpha$ is uniform on $\mathcal{V},$ by the convexity of KL divergence, it then follows that
\eq{
 \kl(\mpr_{X_1|\alpha^*,U_1}||\mpr_{X_1|U_1}) \le \frac{1}{\mathcal{|V|}}\sum_{\alpha\in\mathcal{V}}
\kl(\mpr_{X_1|\alpha^*,U_1}||\mpr_{X_1|\alpha,U_1}) \,.
 }
 For any pair $\alpha^*,\alpha\in\mathcal{V},$ the KL divergence $\kl(\mpr_{X_1|\alpha^*,U_1}||\mpr_{X_1|\alpha,U_1})$ 
 can be at most the KL divergence between a pair of Bernoulli variables with parameters~
 $$
 1-\frac{2+\alpha_i}{4}(4\delta)^{\frac{\kappa+1}{\kappa}}, \quad \text{and}\quad 1-\frac{2+\alpha_j}{4}(4\delta)^{\frac{\kappa+1}{\kappa}},\quad\forall \alpha_i,\alpha_j\in\mathcal{V}.
 $$
By Lemma~\ref{lem:KL} ( setting $p=(4\delta)^{\frac{\kappa+1}{\kappa}}),$ we have  $\kl(\mpr_{X_1|\alpha^*,U_1}||\mpr_{X_1|\alpha,U_1})\le (4\delta)^{\frac{\kappa+1}{\kappa}}.$ 
This complete the proof.

\end{proof}

\begin{lemma}[Lower bound~\ref{thm:cvx_lb1} with $d=1$] 
\label{lem:lb_cointoss1}
Given a constant~$\kappa\in(0,1]$ and a parameter $\alpha^*\in\mathcal{V},$ where $\mathcal{V}=\{-1,+1\},$ the oracle~$\phi$ generates the data sequence $\{X_t\}_{t=1}^T$ where $X_t$ are i.i.d random variables following from the Bernoulli distribution with parameter $1-\frac{2+\alpha^*}{4}(4\delta)^{\frac{\kappa+1}{\kappa}}.$
 Then, for any $\delta\in(0,\frac{1}{8}],$  it holds for any estimator~$\hat{\alpha}\in\mathcal{V}$ based on the data sequence  $\{X_t\}_{t=1}^T$ that
\eq{
 \max_{\alpha^*\in\mathcal{V}}\mpr(\hat\alpha\neq \alpha^*)\ge  \frac{1}{2}\Bigg(1-\sqrt{\frac{(4\delta)^{\frac{\kappa+1}{\kappa}}T}{2}} \Bigg)\,.
}
%Here the probability is taken over the randomness of $\alpha^*$ and $\phi.$
\end{lemma}
\begin{proof}[Proof of Lemma~\ref{lem:lb_cointoss1}]
Set $p=(4\delta)^{\frac{\kappa+1}{\kappa}}.$
Define $\hat\alpha':=1-\frac{2+\hat\alpha}{4}p, {\alpha^*}':=1-\frac{2+\alpha^*}{4}p.$
It then follows that $\mpr(\hat\alpha\neq\alpha^*)=\mpr(\hat\alpha'\neq {\alpha^*}').$
Note that 
\eq{
 \E_{}[|\hat\alpha'- {\alpha^*}'|]=\frac{1}{2}p \mpr_{}(\hat\alpha'\neq {\alpha^*}')\,.
}
Based on the proof of Lemma 4 in~\cite{agarwal2012information}, we have
\eq{
\max_{{\alpha^*}' \in\{1-\frac{1}{4}p,1-\frac{3}{4}p\}} \E_{}[|\hat\alpha'- {\alpha^*}' |]
%\ge& \frac{1}{4}(4\delta)^{\frac{\kappa+1}{\kappa}} \Big( 1-\normtv{\mpr^{+1}-\mpr^{-1}} \Big)\\
\ge& \frac{1}{4}p \Bigg( 1-\frac{1}{2}\sqrt{2T\kl\Big(\mpr^+ || \mpr^- \Big) }\Bigg)\,,
}
where $\mpr^+, \mpr^-$ denote the Bernoulli distributions~$\textsc{ber}(1- \frac{3}{4}p)$ and $\textsc{ber}(1- \frac{1}{4}p)$, respectively.
Combining these two displays with Lemma~\ref{lem:KL} gives
\eq{
\max_{{\alpha^*}' \in\{1-\frac{1}{4}p,1-\frac{3}{4}p\}} \mpr(\hat\alpha'\neq {\alpha^*}' )&= \frac{\max_{{\alpha^*}' \in\{1-\frac{1}{4}p,1-\frac{3}{4}p\}}\E_{}[|\hat\alpha'-  {\alpha^*}' |] }{\frac{1}{2}p}\\
&\ge \frac{1}{2}\Bigg(1-\sqrt{\frac{(4\delta)^{\frac{\kappa+1}{\kappa}}T}{2}} \Bigg)\,.
}
as desired.

\end{proof}

\begin{lemma}[Lower bound~\ref{thm:cvx_lb2} with $d\ge 2$] 
\label{lem:lb_cointoss2}
Suppose the vector~$\alpha^*=(\alpha^*_1,\dots,\alpha^*_d)^\top$ is chosen uniformly at random from the set $\mathcal{V},$ where $\mathcal{V}$ is a subset of the hypercube~$\{-1,+1\}^d$ such that $\ham(\alpha,\tilde{\alpha})=\sum_{i=1}^d \1\{\alpha_i\neq\tilde\alpha_i\}\ge \frac{d}{4}$ for any $\alpha,\tilde{\alpha}\in\mathcal{V}.$  
Set the parameter~
$$
\tilde\alpha^*=(\frac{1}{2}+\alpha_1^*\delta,\dots,\frac{1}{2}+\alpha_d^*\delta)^\top.
$$ 
Given the parameter~$\tilde\alpha^*$, a constant $\delta\in(0,\frac{1}{100}],$ and the time horizon~$T,$ at each round $t=1,\dots,T,$ 
 the oracle~$\phi$ flips a coin with bias~$\frac{1}{T}$  (the probability of the coin landing heads up is $\frac{1}{T}$) at first. 
If the coin has a head, the oracle tosses set of $d$ coins with bias~$\tilde\alpha^*$, and then reveal the outcomes of the $d$ coins. 
{If the coin has a tail, the oracle reveals nothing.}
When $d\ge 2,$ it holds for any estimator~$\hat{\alpha}\in\mathcal{V}$ that
\eq{
\mpr(\hat\alpha\neq \alpha^*)\ge 1-\frac{16d\delta^2+\log 2}{d/8}\,.
}
Here, the probability is taken over the randomness of $\alpha^*$ and $\phi.$
%When $d=1,$  it holds for any estimator~$\hat{\alpha}\in\mathcal{V}$ that
%\eq{
% \mpr_{}(\hat\alpha\neq \alpha^*)\ge 1-\sqrt{\frac{(4\delta)^{\frac{\kappa+1}{\kappa}}T}{2}}\,.
%}
\end{lemma}
\begin{proof}[Proof of Lemma~\ref{lem:lb_cointoss2}]
%Denote the Bernoulli distribution for $i$-th coin by $P_{\tilde\alpha_i}.$ 
Let $U_t\in\{0,1\}$ following the Bernoulli distribution with parameter~$\frac{1}{T}$ be the random variable indicating whether the oracle reveals the information.
Let $X_t:=(X_{t,1},\dots,X_{t,d})^\top$ denote the outcome of oracle's coin toss at time $t$ with the components $X_{t,i}\in\{0,1\}$ denote the outcome for coordinate $i.$
When $U_t=0,$ set $X_{t,i}=-1, i=1,\dots,d.$
By~\cite[Sec 15.3.2, Lemma 4]{pollard2012festschrift} and~\cite[Theorem 1]{scarlett2019introductory}, if the parameter $\alpha^*$ is uniform on $\mathcal{V},$ it holds for any estimator $\hat\alpha\in\mathcal{V}$ that
\eq{
\mpr(\hat\alpha\neq \alpha^*)\ge 1-\frac{I(\{U_t,X_t\}_{t=1}^T;\alpha^*)+\log 2}{\log |\mathcal{V}|}\,,
}
where $I(\{U_t,X_t\}_{t=1}^T;\alpha^*)$ denotes the mutual information between the data sequence~$\{U_t,X_t\}_{t=1}^T$ and $\alpha^*$.
By the Varshamov-Gilbert bound, 
there exists such a packing set~$\mathcal{V}\subseteq \{-1,+1\}^d$ with $|\mathcal{V}|\ge \exp(\frac{d}{8})$ satisfies
$\ham(\alpha,\tilde{\alpha})=\sum_{i=1}^d \1\{\alpha_i\neq\tilde\alpha_i\}\ge \frac{d}{4}$ for any $\alpha,\tilde{\alpha}\in\mathcal{V}.$  
It suffice to show that $I(\{U_t,X_t\}_{t=1}^T; \alpha^*) \le 16d\delta^2.$
By the independent and identically distributed the sampling, we have
\eq{
I(\{U_t,X_t\}_{t=1}^T; \alpha^*)=\sum_{t=1}^T I\big( (U_1,X_1); \alpha^*\big)=TI\big( (U_1,X_1); \alpha^*\big)\,.
}
By chain rule of mutual information and the sampling scheme, it holds that 
\eq{
I\big( (U_1,X_1); \alpha^*\big)=I(X_1;\alpha^*|U_1)+I(\alpha^*;U_1)\,.
}
Note that $U_1$ is sampled independent of $\alpha^*,$  this implies $I(\alpha^*;U_1)=0.$
It remains to show that $I(X_1;\alpha^*|U_1)\le \frac{1}{T}16d\delta^2.$
By definition of the conditional mutual information, and the factorization~$\mpr_{X_1,\alpha^*|U_1}=\mpr_{\alpha^*|U_1}\mpr_{X_1|\alpha^*,U_1},$ it holds that 
\eq{
I(X_1;\alpha^*|U_1)
&= \E_{U_1}\big[ \kl(\mpr_{X_1|\alpha^*,U_1}||\mpr_{X_1|U_1}) \big]\,.
}
Assume a random vector $\alpha$ is uniform on $\mathcal{V},$ by the convexity of KL divergence, it then follows that
\eq{
 \kl(\mpr_{X_1|\alpha^*,U_1}||\mpr_{X_1|U_1}) \le \frac{1}{\mathcal{|V|}}\sum_{\alpha\in\mathcal{V}}
\kl(\mpr_{X_1|\alpha^*,U_1}||\mpr_{X_1|\alpha,U_1}) \,.
 }
 Combing these two display with fact that $U_1\sim \textsc{ber}(\frac{1}{T})$ gives
 \eq{
 I(X_1;\alpha^*|U_1)
 \le & \frac{1}{\mathcal{|V|}}\sum_{\alpha\in\mathcal{V}} \E_{U_1}   \kl(\mpr_{X_1|\alpha^*,U_1}||\mpr_{X_1|\alpha,U_1}) \big] \\
  \le & \frac{1}{T}   \frac{1}{\mathcal{|V|}}\sum_{\alpha\in\mathcal{V}}  \kl(\mpr_{X_1|\alpha^*,U_1=1}||\mpr_{X_1|\alpha,U_1=1}) \\
&\quad  + \Big( 1-\frac{1}{T} \Big)  \frac{1}{\mathcal{|V|}}\sum_{\alpha\in\mathcal{V}}    \kl(\mpr_{X_1|\alpha^*,U_1=0}||\mpr_{X_1|\alpha,U_1=0}) \\
= &  \frac{1}{T}   \frac{1}{\mathcal{|V|}}\sum_{\alpha\in\mathcal{V}}  \kl(\mpr_{X_1|\alpha^*,U_1=1}||\mpr_{X_1|\alpha,U_1=1})  \,.
 }
 For any pair $\alpha^*,\alpha\in\mathcal{V},$ the KL divergence $\kl(\mpr_{X_1|\alpha^*,U_1}||\mpr_{X_1|\alpha,U_1})$ 
 can be at most the KL divergence between $d$ independent pairs of Bernoulli variables with parameters~$\frac{1}{2}+\delta$ and $\frac{1}{2}-\delta.$
By Lemma 3 in~\cite{agarwal2012information}, it holds that
\eq{
\label{eq:kl}
\kl(\mpr_{X_1|\alpha^*,U_1=1}||\mpr_{X_1|\alpha,U_1=1})  \le 16d\delta^2\,.
}
Thus, we have
\eq{
I(X_1;\alpha^*|U_1)\le \frac{1}{T}16d\delta^2
}
as desired.

\end{proof}

\begin{lemma}[Lower bound~\ref{thm:cvx_lb2} with $d=1$] 
\label{lem:lb_cointoss4}
Given a parameter $\alpha^*\in\mathcal{V},$ where $\mathcal{V}=\{-1,+1\},$ a constant $\delta\in(0,\frac{1}{100}],$ and the time horizon~$T.$  
At each round $t=1,\dots,T,$  the oracle~$\phi$ flips a coin with probability of getting heads being $\frac{1}{T}.$
If the coin lands on heads, the oracle tosses a coin with bias~$\frac{1}{2}+\alpha^*\delta$ and then reveal the outcome.
{If the coin has a tail, the oracle reveals nothing.}
 Then, it holds for any estimator~$\hat{\alpha}\in\mathcal{V}$ that
\eq{
 \max_{\alpha^*\in\mathcal{V}}\mpr(\hat\alpha\neq \alpha^*)\ge  1-\sqrt{8\delta^2} \,.
}
%Here the probability is taken over the randomness of $\alpha^*$ and $\phi.$
\end{lemma}
\begin{proof}[Proof of Lemma~\ref{lem:lb_cointoss4}]
%Set $p=(4\delta)^{\frac{\kappa+1}{\kappa}}.$
Define $\hat\alpha':=\frac{1}{2}+\hat\alpha\delta , {\alpha^*}':=\frac{1}{2}+\alpha^*\delta.$
It then follows that $\mpr(\hat\alpha\neq\alpha^*)=\mpr(\hat\alpha'\neq {\alpha^*}').$
Note that 
\eq{
 \E_{}[|\hat\alpha'- {\alpha^*}'|]=2\delta \mpr_{}(\hat\alpha'\neq {\alpha^*}')\,.
}
Based on the proof of~\cite[Lemma 4]{agarwal2012information} and display~\eqref{eq:kl}, we have
\eq{
\max_{{\alpha^*}' \in\{\frac{1}{2}+\delta,\frac{1}{2}-\delta \}} \E_{}[|\hat\alpha'- {\alpha^*}' |]
%\ge& \frac{1}{4}(4\delta)^{\frac{\kappa+1}{\kappa}} \Big( 1-\normtv{\mpr^{+1}-\mpr^{-1}} \Big)\\
\ge& 2\delta \big( 1- \sqrt{ 8\delta^2} \big)\,,
}
%where $\mpr^+, \mpr^-$ denote the Bernoulli distributions~$Bernoulli(\frac{1}{2}+\delta)$ and $Bernoulli(\frac{1}{2}-\delta)$, respectively.
Combining these two displays gives
\eq{
&\max_{{\alpha^*}' \in\{\frac{1}{2}+\delta,\frac{1}{2}-\delta \}} \mpr(\hat\alpha'\neq {\alpha^*}' )
%=& \frac{\max_{{\alpha^*}' \in\{ \frac{1}{2}+\delta,\frac{1}{2}-\delta \}}\E_{}[|\hat\alpha'-  {\alpha^*}' |] }{2\delta}\\
\ge 1-\sqrt{8\delta^2} 
}
as desired.

\end{proof}

\subsection{Proofs of minimax lower bounds}

We are now ready to prove the minimax lower bounds. 
In this section, we use the subscript~$i$ to denote the $i$-th digit of a vector and use the superscript $t$ to denote the time index.
For instance, given the $t$-th iterate~$x^t\in\rd,$  $x^t_i$ represents the $i$-th element of $x^t.$

\medskip

\begin{proof}[Proof of Theorem~\ref{thm:cvx_lb}]

\medskip

\noindent\textbf{Proof of lower bound~\eqref{thm:cvx_lb1}}\\
At first, we consider the special case $\S=S_\infty(R).$
The proof consists four steps.
We first construct a subclass of functions parametrized by a subset of the vertices of a $d$-dimensional hypercube with finite cardinality. 
Then, we construct a stochastic oracle based on Bernoulli random variables, each of which corresponds to the parameters of the constructed function in the previous step.
Next, we convert the parameter estimation to the stochastic optimization problem by showing that optimizing any function in this subclass to certain tolerance requires identifying the hypercube vertices.
%This can be treated as a multiway hypothesis test based on observations provided by querying the stochastic oracle $T$ rounds.
Finally, we employ Fano types of inequality to lower bound the probability of misspecification error, along with the results obtained in the previous steps, to finish the proof.
The four mentioned steps now read in detail.

\noindent \textbf{1. Construct a subclass of functions}\\
Assume $\mathcal{V}\subseteq\{-1,+1\}^d$ is a subset of the hypercube such that 
$$
\ham(\alpha,\tilde{\alpha})=\sum_{i=1}^d \1\{\alpha_i\neq\tilde\alpha_i\}\ge \frac{d}{4},
$$ 
for any $\alpha,\tilde{\alpha}\in\mathcal{V}.$ 
Given a vector~$\alpha=(\alpha_1,\dots,\alpha_d)^\top\in\mathcal{V},$ consider the convex function~$g_{\alpha}(x):\S\to\R$ defined via  
\eq{
 g_{\alpha}(x) := \frac{L}{d}\sum_{i=1}^d \frac{2+\alpha_i}{4}\delta  \Big\{ (1+\alpha_i)|x_i+R| +(1-\alpha_i)|x_i-R| \Big\}
}
with $\delta \in(0,\frac{1}{8}]$.
Define the function $h(\alpha, x)$ via
\eq{
h:\{-1,+1\}\times \S&\to [0,\infty)\\
(\alpha, x) &\mapsto \frac{1}{2}\big [(1+\alpha)|x+R| +(1-\alpha)|x-R| \big]\,.
}
%$h(\alpha, x):=\frac{1}{2}\big [(1+\alpha)|x+R| +(1-\alpha)|x-R| \big].$
We then have
$
|\grad h(\alpha,x)|\le 1.
$
Hence, it holds for any $q\in[1,1+\kappa]$ that
 \eq{
 \norm{\grad g_{\alpha}(x)}_q \le \frac{L}{d} \Big( \sum_{i=1}^d \big (\frac{2+\alpha_i}{2}\delta|\grad h_i(\alpha_i,x_i)| \big)^q \Big)^{\frac{1}{q}}\le L\,.
 }
%\eq{
%|g_{\alpha}(x)-g_{\alpha}(y)|\le \frac{L}{d}\sum_{i=1}^d \Bigg\{ \frac{2+\alpha_i}{2}\delta |\alpha_i| |x_i-y_i|  \Bigg\}\le L\,\normsup{x-y}\le L\,\norm{x-y}_{q^*}\,,
%}
%where $q^*\in[2,\infty].$
This implies  $g_{\alpha}(x)$ is $L$-Lipschitz with respect to ${q^*}$ norm, where $q^*$ satisfies $\frac{1}{q}+\frac{1}{q^*}=1.$
It follows that $g_{\alpha}\in\mathcal{H}_{cvx},\forall\alpha\in\V.$
Define the function class $\mathcal{G}(\delta):=\{g_{\alpha}:\alpha\in\mathcal{V}\}.$ 
%By the Varshamov-Gilbert bound (see~\cite[Lemma 4]{yu1997assouad}, for example), 
%there exists such a packing set~$\mathcal{V}$ whose cardinality satisfies $|\mathcal{V}|\ge \exp(\frac{d}{8})$.
%It then follows that $|\mathcal{V}|\ge (2/\sqrt{e})^{d/2}$.
Set 
\begin{align*}
p:= \big(4\delta\big)^{\frac{\kappa+1}{\kappa}},\quad\text{and}\quad \Lambda:=\frac{1}{2}p^{-\frac{1}{1+\kappa}},\quad\text{where}\quad  \kappa\in(0,1].
\end{align*} 
It then follows that
\eq{
p\in\left(0,{1}/{2}\right],\quad  p\Lambda=2\delta\in\left(0,{1}/{4}\right]\quad\text{and}\quad g_{\alpha}(x)=\frac{L}{d}\sum_{i=1}^d  \frac{2+\alpha_i}{4}p\Lambda h(\alpha_i,x_i)  \,.
}
\noindent\textbf{2. Construct an oracle}

Now, we describe the stochastic first order oracle~$\phi$ which satisfies the conditions stated in Assumption~\ref{as:sfo}.
%\red{Instead, set  $p:= \Big(\frac{d}{n}\Big)^{\frac{(1-a)\kappa+1}{1+\kappa}} \delta, L:=( \frac{d}{n})^{-\frac{(1-a)\kappa+1}{(\kappa+1)^2}}$ with $a\in(0,1]$, and $\delta<\min\{ \Big(\frac{d}{n}\Big)^{-\frac{(1-a)\kappa+1}{1+\kappa}} ,L,1\}$}\\
Given a vector~$\alpha\in\mathcal{V}$, consider the oracle $\phi$ that returns noisy value and gradient sample as following for $t=1.\dots.T$:
 \\
 \\
\noindent 1). Pick an index $i_t\in\{1,\dots,d\}$ uniformly.

\noindent 2). Draw $b_{i_t}\in\{0,1\}$ according to $\textsc{ber}\big(1-\frac{2+\alpha_{i_t}}{4}p \big) .$

\noindent 3). For the given input~$x\in\S$, return the function value $\hat g_{\alpha}(x) = L(1-b_{i_t})  \Lambda h(\alpha_{i_t},x)$ and its subgradient.
\\
\\
Now, we verify the constructed oracle satisfies the conditions stated in Assumption~\ref{as:sfo}.
Note that 
\eq{
\E[\hat g_{\alpha}(x^{t})|\filtration_t]=\frac{L}{d} \sum_{i=1}^d  \frac{2+\alpha_i}{4}p\Lambda h(\alpha_i,x^t_i) =g_{\alpha}(x^t)\,.
}
Moreover, note that
\eq{
\frac{\partial }{\partial x_i}  L (1-b_{i})  \Lambda h(\alpha_i,x_i) 
=  L (1-b_{i})  \Lambda \grad h(\alpha_i,x_i)  \,.
}
We then find
\eq{
 \E[\grad \hat g_\alpha(x^t)|\filtration_t ]=\grad g_{\alpha}(x^t)\,,
}
and
\eq{
 \E[ \norm{\grad \hat g_{\alpha}(x^t)}_q^{1+\kappa}|\filtration_t]\le \frac{L^{1+\kappa}}{d}\sum_{i=1}^d \Lambda^{\kappa+1} \frac{2+\alpha_i}{4}p \le L^{1+\kappa},~~~\forall q\in [1,1+\kappa]\,.
}

\noindent \textbf{3. Optimizing well is equivalent to function identification}

In this step, we employ the same quantification of the function separation as in~\cite{agarwal2012information}.
Define the discrepancy measure between two functions~$f,g$ over the same domain $\S$ as
\eq{
\rho(f,g):=\inf_{x\in \S}[f(x)+g(x)-f(x_f^*)-g(x_g^*)]\,.
}
Given the function class $\mathcal{G}(\delta),$ define $\psi(\mathcal{G}(\delta)):=\min_{\alpha\neq\beta\in \mathcal{V}} \rho(g_{\alpha},g_{\beta}).$

%\textbf{Lemma 2} \red{(upper bound)}
Given an vector $\alpha^*\in\mathcal{V},$ we have corresponding function~$g_{\alpha^*}.$
Suppose the method~$M_T$ makes $T$ queries to the oracle~$\phi,$ and thus obtains the information sequence~$\{\phi(x^1; g_{\alpha^*}),\dots, \phi (x^T; g_{\alpha^*} )\},$
denoted by  $\phi(x^1_T; g_{\alpha^*}).$
%denote the data  sequence $\{\phi(x^1; g_{\alpha^*}),\dots, \phi (x^T; g_{\alpha^*} )\}$ by $\phi(x^1_T; g_{\alpha^*}),$ where $g_\alpha^*\in \mathcal{G}(\delta)$ is an unknown but fixed function. 
By~\cite[Lemma 2]{agarwal2012information}, for any method ${M}_T\in\M_T$ one can construct a hypothesis test $\hat\alpha: \phi(x^1_T; g_{\alpha^*}) \to \mathcal{V}$ such that 
\eq{
\mpr_{\phi} \big( \hat\alpha(M_T)\neq \alpha^*  \big) 
\le 
\mpr_{\phi} \Big( \epsilon({M}_T, g_{\alpha^*}, \S,\phi) \ge \frac{\psi(\mathcal{G}(\delta))}{3} \Big) ,\forall \alpha^*\in\V \,.
}
This implies
\eq{\label{eq:opt_hypotest}
  \frac{1}{|\V|}\sum_{\alpha^*\in\V} \mpr_{\phi} \big( \hat\alpha(M_T)\neq \alpha^*  \big) 
 \le & \frac{1}{|\V|}\sum_{\alpha^*\in\V} \mpr_{\phi} \Big( \epsilon({M}_T, g_{\alpha^*}, \S,\phi) \ge \frac{\psi(\mathcal{G}(\delta))}{3} \Big)\,.
}
Moreover, by the definition of $\epsilon^*_T(\mathcal{H}_{cvx},\S,\phi),$ we have
\eq{
\epsilon^*_T(\mathcal{H}_{cvx},\S,\phi) 
\ge & \inf_{M_T\in\mathcal{M}_T}\sup_{\alpha^*\in\V}\E_\phi[\epsilon({M}_T, g_{\alpha^*}, \S,\phi)]\,.
}
By Markov's inequality, we then find 
\eq{
 \E_\phi[\epsilon({M}_T, g_{\alpha^*}, \S,\phi)]
 \ge&  \frac{\psi(\mathcal{G}(\delta))}{3} \mpr_{\phi} \Big( \epsilon({M}_T, g_{\alpha^*}, \S,\phi) >  \frac{\psi(\mathcal{G}(\delta))}{3}  \Big)\,.
}
Combining this with previous display provides us with
\eq{
\epsilon^*_T(\mathcal{H}_{cvx},\S,\phi) 
\ge & \inf_{M_T\in\mathcal{M}_T}\sup_{\alpha^*\in\V}  \frac{\psi(\mathcal{G}(\delta))}{3} \mpr_{\phi} \Big( \epsilon({M}_T, g_{\alpha^*}, \S,\phi) >  \frac{\psi(\mathcal{G}(\delta))}{3}  \Big)\\
\ge &  \frac{\psi(\mathcal{G}(\delta))}{3} \inf_{M_T\in\mathcal{M}_T} \frac{1}{|\V|}\sum_{\alpha^*\in\V}  \mpr_{\phi} \Big( \epsilon({M}_T, g_{\alpha^*}, \S,\phi) >  \frac{\psi(\mathcal{G}(\delta))}{3}  \Big)\,.
}
Plugging inequality~\eqref{eq:opt_hypotest} into it gives
\eq{
\epsilon^*_T(\mathcal{H}_{cvx},\S,\phi) 
\ge&   \frac{\psi(\mathcal{G}(\delta))}{3} \inf_{M_T\in\mathcal{M}_T}   \frac{1}{|\V|}\sum_{\alpha^*\in\V} \mpr_{\phi} \big( \hat\alpha(M_T)\neq \alpha^*  \big)\,,
}
which implies
\eq{
\label{eq:lb_aux}
\epsilon^*_T(\mathcal{H}_{cvx},\S,\phi)  \ge&  \frac{\psi(\mathcal{G}(\delta))}{3} \inf_{\hat\alpha\in\V}  \frac{1}{|\V|}\sum_{\alpha^*\in\V} \mpr_{\phi} \big( \hat\alpha(M_T)\neq \alpha^*  \big) \,.
}
In the next step, we will finish the proof by providing the lower bounds for the discrepancy $\psi(\mathcal{G}(\delta))$ and the probability $ \inf_{\hat\alpha\in\V}  \frac{1}{|\V|}\sum_{\alpha^*\in\V} \mpr_{\phi} \big( \hat\alpha(M_T)\neq \alpha^*  \big)$ with some specific choice of $\delta.$\\

\noindent\textbf{4. Complete the proof}\\
Note that 
the minimizer of $g_\alpha(x)$ is $x_{\alpha}^*=-{R}\alpha,$
and $\min_{x\in\S} g_\alpha(x)=0.$
Then, it holds that
\eq{
&g_{\alpha}(x)+g_{\beta}(x)-g_{\alpha}(x_{\alpha}^*)-g_{\beta}(x_{\beta}^*)\\
=&\frac{L}{d}\sum_{i=1}^d\Bigg\{  \frac{2+\alpha_i}{4}p\Lambda h(\alpha_i, x_i)   +  \frac{2+\beta_i}{4}p\Lambda h(\beta_i,x_i)  
\Bigg\}\\
=:& \sum_{i=1}^d I(x_i;\alpha_i,\beta_i)\,,
}
where $I(x_i;\alpha_i,\beta_i):= \frac{L}{d}  \Big\{ \frac{2+\alpha_i}{4}p\Lambda h(\alpha_i,x_i)   +  \frac{2+\beta_i}{4}p\Lambda h(\beta_i,x_i) \Big\}  .$
When $\alpha_i=\beta_i,$ it holds that  $\min_{x\in\S}I(x;\alpha_i,\beta_i)=0.$
When $\alpha_i\neq\beta_i,$ it holds that 
\eq{
I(x_i;\alpha_i,\beta_i)= \frac{L\delta}{d}\Big\{ \frac{3}{2}|x_i+R| +\frac{1}{2}|x_i-R|\Big\}\,,
}
it then follows that
 $\min_{x\in\S} I(x;\alpha_i,\beta_i)=\frac{L\delta}{d}R.$
Thus, we obtain 
\eq{
\rho(g_{\alpha},g_{\beta}) = \frac{RL\delta}{d }\ham(\alpha,\beta)\ge \frac{RL\delta}{d} \frac{d}{4}= \frac{RL\delta}{4} \,,
}
which implies 
\eq{
\psi(\mathcal{G}(\delta))\ge \frac{RL\delta}{4}\,.
}
Recall that we obtain the following in step 3
\eq{
\epsilon^*_T(\mathcal{H}_{cvx},\S,\phi) 
\ge&  \frac{\psi(\mathcal{G}(\delta))}{3} \inf_{\hat\alpha\in\V}  \frac{1}{|\V|}\sum_{\alpha^*\in\V} \mpr_{\phi} \big( \hat\alpha(M_T)\neq \alpha^*  \big) \,.
}
Combining the previous two displays gives
\eq{
\label{eq:fano00}
\epsilon^*_T(\mathcal{H}_{cvx},\S,\phi) 
\ge& \frac{RL\delta}{12}  \inf_{\hat\alpha\in\V}  \frac{1}{|\V|}\sum_{\alpha^*\in\V} \mpr_{\phi} \big( \hat\alpha(M_T)\neq \alpha^*  \big) \,.
}
When $d>8,$ invoking Lemma~\ref{lem:lb_cointoss} yields
\eq{
\label{eq:fano11}
\epsilon^*_T(\mathcal{H}_{cvx},\S,\phi) 
\ge& \frac{{R}L\delta}{12} \Big( 1-\frac{(4\delta)^{\frac{\kappa+1}{\kappa}}T+\log 2}{d/8} \Big)\,.
}
Let $T\ge d$ with $d\ge 9,$ and set $\delta:=\frac{1}{32}\Big(\frac{d}{T}\Big)^{\frac{\kappa}{1+\kappa}}.$
It then follows that
\eq{
0< \delta \le\frac{1}{8}\,,
}
and
\eq{
\frac{(4\delta)^{\frac{\kappa+1}{\kappa}}T+\log 2}{d/8}\le \frac{8}{8^{\frac{\kappa+1}{\kappa}}}+\frac{8\log 2}{d}\le \frac{3}{4}\,.
}
Plugging these into display~\eqref{eq:fano11} then gives
\eq{
\epsilon^*_T(\mathcal{H}_{cvx},\S,\phi) 
\ge& \frac{1}{1536} RL \Big(\frac{d}{T}\Big)^{\frac{\kappa}{1+\kappa}}\,.
}

When $d<9,$ we restrict to the case where $d=1.$
The lower bounds corresponding $1< d \le 8$ can be established based on the case of $d=1.$
Combining the lower bound in Lemma~\ref{lem:lb_cointoss1} with the display~\eqref{eq:fano00} gives
\eq{
\label{eq:lecam}
\epsilon^*_T(\mathcal{H}_{cvx},\S,\phi) 
\ge& \frac{{R}L\delta}{12}  \frac{1}{2}\Bigg(1-\sqrt{\frac{(4\delta)^{\frac{\kappa+1}{\kappa}}T}{2}} \Bigg) \,.
}
Set $\delta:=\frac{1}{32}T^{-\frac{\kappa}{1+\kappa}}.$
Then we have $\delta\in(0,{1}/{8}]$ and 
\eq{
\sqrt{\frac{(4\delta)^{\frac{\kappa+1}{\kappa}}T}{2}}
 \le\sqrt{\Big(\frac{1}{8}\Big)^{\frac{1+\kappa}{\kappa}}\frac{1}{2}}
 \le \frac{1}{10}\,. 
}
Combing these two displays yields
\eq{
\epsilon^*_T(\mathcal{H}_{cvx},\S,\phi) 
\ge& \frac{9}{7680} RL\Big(\frac{1}{T}\Big)^{\frac{\kappa}{1+\kappa}}\,.
}
This completes the proof for the special case $\S=S_\infty(R).$
 Note that the Lipschitz constant of $g_\alpha$ does not depend on $\S$, $x^*_\alpha=\argmin_{x\in S} g_\alpha (x)\in\S,$ and thus the preceding proof goes through when $\S\supseteq  S_{\infty}(R).$
 Hence, the desired general claim follows.

\bigskip

\noindent\textbf{Proof of the lower bound~\eqref{thm:cvx_lb2}}

The proof strategy is similar to the proof of the lower bound~\eqref{thm:cvx_lb1}, but with a  different function class and the first-order oracle.
At first, we consider the special case $\S=S_\infty(R).$
The proof consists four steps as follows.
\\
\\
\noindent \textbf{1. Construct a subclass of functions}\\
Assume $\mathcal{V}\subseteq\{-1,+1\}^d$ is a subset of the hypercube such that 
$$
\ham(\alpha,\tilde{\alpha})=\sum_{i=1}^d \1\{\alpha_i\neq\tilde\alpha_i\}\ge \frac{d}{4},
$$ 
for any $\alpha,\tilde{\alpha}\in\mathcal{V}.$ 
Given the time horizon~$T$, and a vector~$\alpha=(\alpha_1,\dots,\alpha_d)^\top\in\mathcal{V},$ we consider the convex function~$g_{\alpha}(x):\S\to\R$ defined via  
\eq{
 g_{\alpha}(x) := \frac{L}{T^{\frac{\kappa}{1+\kappa}}d^{\frac{1}{q}}}\sum_{i=1}^d \Big\{ \Big(\frac{1}{2}+\alpha_i\delta\Big)|x_i+R| + \Big(\frac{1}{2}-\alpha_i\delta\Big) |x_i-R| \Big\}
}
with $\delta \in(0,{1}/{100}]$.
Define the function $h(\alpha, x)$ via
\eq{
h:\{-1,+1\}\times \S&\to [0,\infty)\\
(\alpha, x) &\mapsto \big(\frac{1}{2}+\alpha\delta\big)|x+R| + \big(\frac{1}{2}-\alpha\delta\big) |x-R| \,.
}
% $h_i(\alpha,x):=\big(\frac{1}{2}+\alpha\delta\big)|x+R| + \big(\frac{1}{2}-\alpha\delta\big) |x-R| .$
We then have
$
|\grad h(\alpha,x)|\le 1.
$
Hence, it holds for any $q\in[1,\infty]$ that
 \eq{
 \norm{\grad g_{\alpha}(x)}_q 
 \le \frac{L}{T^{\frac{\kappa}{1+\kappa}}d^{\frac{1}{q}}} \Big(\sum_{i=1}^d |\grad h_i|^q \Big)^{\frac{1}{q}} 
 \le L\,.
 }
This implies  $g_{\alpha}(x)$ is $L$-Lipschitz with respect to ${q^*}$ norm, where $q^*$ satisfies $\frac{1}{q}+\frac{1}{q^*}=1$.
It follows that $g_{\alpha}\in\mathcal{H}_{cvx},\forall\alpha\in\V.$
Define the function class $\mathcal{G}(\delta):=\{g_{\alpha}:\alpha\in\mathcal{V}\}.$ 
%Set $p:= \big(4\delta\big)^{\frac{\kappa+1}{\kappa}}, \Lambda:=\frac{1}{2}p^{-\frac{1}{1+\kappa}},$ where  $\kappa\in(0,1].$
%It then follows that $p\in(0,\frac{1}{2}], p\Lambda=2\delta\in(0,\frac{1}{4}],$ and
%\eq{
%g_{\alpha}(x)=\frac{L}{d^{\frac{1}{q}}}\sum_{i=1}^d  \frac{2+\alpha_i}{4}p\Lambda h_i(x_i)  \,.
%}
\\
\\
\noindent\textbf{2. Construct an oracle}

Now, we describe the stochastic first order oracle~$\phi$ which satisfies the conditions stated in Assumption~\ref{as:sfo}.
%\red{Instead, set  $p:= \Big(\frac{d}{n}\Big)^{\frac{(1-a)\kappa+1}{1+\kappa}} \delta, L:=( \frac{d}{n})^{-\frac{(1-a)\kappa+1}{(\kappa+1)^2}}$ with $a\in(0,1]$, and $\delta<\min\{ \Big(\frac{d}{n}\Big)^{-\frac{(1-a)\kappa+1}{1+\kappa}} ,L,1\}$}\\
Given the time horizon~$T$, and a vector~$\alpha\in\mathcal{V}$, consider the oracle $\phi$ that returns noisy value and gradient sample as following for $t=1,\dots,T$:
 \\
 \\
\noindent 1). Draw $Y_t\in\{0,1\}$ according to $\textsc{ber}\Big(\frac{1}{T}\Big).$ 

\noindent 2a). When $Y_t=1$, draw $b_{i}\in\{0,1\}$ according to $\textsc{ber}\big( \frac{1}{2}+\alpha_{i}\delta \big),i=1,\dots,d.$  
For the given input~$x\in\S$, return the function value 
$$
\hat g_{\alpha}(x) = LT^{\frac{1}{1+\kappa}}d^{-\frac{1}{q}} \sum_{i=1}^d \big\{b_{i}|x_{i}+R|+(1-b_{i})|x_{i}-R| \big\}
$$ 
and its subgradient.

\noindent 2b). When $Y_t=0,$ for any input $x\in\S$, return $\hat g_{\alpha}(x) =0$ and its subgradient.
\\
\\
Now, we verify the conditions in Assumption~\ref{as:sfo} for the constructed oracle.
It is obvious that
\eq{
\E[\hat g_{\alpha}(x^t)|\filtration_t]= g_{\alpha}(x^t)\,.
}
and
\eq{
\E[ \grad \hat g_{\alpha}(x^t)|\filtration_t]= \grad g_{\alpha}(x^t)\,.
}
Moreover, it holds that
\eq{
\frac{\partial}{\partial x} \Big (b_i|x+R|+(1-b_i)|x-R| \Big) \le 1\,.
}
It then follows that 
\eq{
\E \big[\norm{\grad \hat g_\alpha(x^t)}_q^{1+\kappa}|\filtration_t \big]
\le \frac{1}{T} L^{1+\kappa} T d^{-\frac{1+\kappa}{q}} d^{\frac{1+\kappa}{q}}
=L^{1+\kappa}\,.
}

\noindent \textbf{3. Optimizing well is equivalent to function identification}
%\textbf{Lemma 2} \red{(upper bound)}
%Given an vector $\alpha^*\in\mathcal{V},$ we have corresponding function~$g_{\alpha^*}.$

\noindent In this step, we employ the same quantification of the function separation as in step 3 of the proof of Theorem~\ref{thm:cvx_lb},
where the discrepancy measure between two functions~$f,g$ over the same domain $\S$ is
\eq{
\rho(f,g)=\inf_{x\in \S}[f(x)+g(x)-f(x_f^*)-g(x_g^*)]\,.
}
Given the function class $\mathcal{G}(\delta),$ define $\psi(\mathcal{G}(\delta)):=\min_{\alpha\neq\beta\in \mathcal{V}} \rho(g_{\alpha},g_{\beta}).$
Invoking display~\eqref{eq:lb_aux},  we have
\eq{
\epsilon^*_T(\mathcal{H}_{cvx},\S,\phi) 
%\ge&   \frac{\psi(\mathcal{G}(\delta))}{3} \inf_{M_T\in\mathcal{M}_T}   \frac{1}{|\V|}\sum_{\alpha^*\in\V} \mpr_{\phi} \big( \hat\alpha(M_T)\neq \alpha^*  \big)\\
\ge&  \frac{\psi(\mathcal{G}(\delta))}{3}  \inf_{\hat\alpha\in\V}  \frac{1}{|\V|}\sum_{\alpha^*\in\V} \mpr_{\phi} \big( \hat\alpha(M_T)\neq \alpha^*  \big) \,.
}
In the next step, we will finish the proof by providing the lower bounds for the discrepancy $\psi(\mathcal{G}(\delta))$ and the probability $ \inf_{\hat\alpha\in\V}     \frac{1}{|\V|}\sum_{\alpha^*\in\V} \mpr_{\phi} \big( \hat\alpha(M_T)\neq \alpha^*  \big)$ with some specific choice of $\delta.$\\

\noindent\textbf{4. Complete the proof}\\
%Note that 
%the minimizer of $g_\alpha(x)$ is $x_{\alpha}^*=-{R}\alpha,$
%and $\min_{x\in\S} g_\alpha(x)=0.$
%%So, we have
%%\eq{
%%g_{\alpha}(x_{\alpha}^*)= -{R}\frac{L}{d}\sum_{i=1}^d \Big\{\frac{2+\alpha_i}{4}p\Lambda \Big\}\,.
%%}
%Then, it holds that
%\eq{
%&g_{\alpha}(x)+g_{\beta}(x)-g_{\alpha}(x_{\alpha}^*)-g_{\beta}(x_{\beta}^*)\\
%=&\frac{L}{d^{\frac{1}{q}}}\sum_{i=1}^d\Bigg\{  \frac{2+\alpha_i}{4}p\Lambda h_i(x_i)   +  \frac{2+\beta_i}{4}p\Lambda h_i(x_i)  
%\Bigg\}\\
%=:& \sum_{i=1}^d I_i\,,
%}
%where $I_i:= \frac{L}{d^{\frac{1}{q}}}  \Big\{ \frac{2+\alpha_i}{4}p\Lambda h_i(x_i)   +  \frac{2+\beta_i}{4}p\Lambda h_i(x_i) \Big\}  .$
%% + \frac{2+\alpha_i}{4}p\Lambda {R} + \frac{2+\beta_i}{4}p\Lambda {R}
%%\Big\} .$
%When $\alpha_i=\beta_i,$ it holds that  $\min_{\alpha_i=\beta_i}I_i=0.$
%When $\alpha_i\neq\beta_i,$ it holds that 
%\eq{
%I_i= \frac{L\delta}{d^{\frac{1}{q}}}\Big\{ \frac{3}{2}|x_i+R| +\frac{1}{2}|x_i-R|\Big\}\,,
%}
%it then follows that
% $\min_{\alpha_i\neq\beta_i} I_i=\frac{L\delta}{d^{\frac{1}{q}}}R.$
We note that the function~$g_\alpha(x)$ is a specification of the function class considered in part (a) of ~\cite[Theorem 1]{agarwal2012information}
\eq{
g_{\alpha}(x) := \frac{c}{d}\sum_{i=1}^d \Big\{ \Big(\frac{1}{2}+\alpha_i\delta\Big)|x_i+R| + \Big(\frac{1}{2}-\alpha_i\delta\Big) |x_i-R| \Big\}
}
by setting $c=\frac{Ld}{T^{\frac{\kappa}{1+\kappa}}d^{\frac{1}{q}}}.$
By the last display in the proof of Theorem 1 in~\cite{agarwal2012information}, it holds that $\rho(g_{\alpha},g_{\beta})\ge \frac{cR\delta}{2},\forall\alpha\neq \beta\in\mathcal{V}.$
We then have 
\eq{
\psi(\mathcal{G}(\delta))\ge  \frac{1}{2}R\delta LT^{-\frac{\kappa}{1+\kappa}} d^{1-\frac{1}{q}}  \,.
}
%\ly{lower bound of the discrepancy so that step 2 follows and relate $\epsilon$ to $\delta$}
%then Step 2 (1/3 upper bound) follows. 
%Set $\epsilon:=\frac{d^{\frac{1}{q^*}}{R}L\delta}{36}.$
Recall that we obtain the following in step 3
\eq{
\epsilon^*_T(\mathcal{H}_{cvx},\S,\phi) 
\ge&  \frac{\psi(\mathcal{G}(\delta))}{3}  \inf_{\hat\alpha\in\V}  \frac{1}{|\V|}\sum_{\alpha^*\in\V} \mpr_{\phi} \big( \hat\alpha(M_T)\neq \alpha^*  \big) \,.
}
Combining the previous two displays gives
\eq{
\label{eq:fano0}
\epsilon^*_T(\mathcal{H}_{cvx},\S,\phi) 
\ge& \frac{1}{6} R\delta LT^{-\frac{\kappa}{1+\kappa}} d^{1-\frac{1}{q}}   \inf_{\hat\alpha\in\V}  \frac{1}{|\V|}\sum_{\alpha^*\in\V} \mpr_{\phi} \big( \hat\alpha(M_T)\neq \alpha^*  \big) \,.
}
When $d>8,$ invoking Lemma~\ref{lem:lb_cointoss2} yields 
\eq{
\label{eq:fano1}
\epsilon^*_T(\mathcal{H}_{cvx},\S,\phi) 
\ge& \frac{1}{6} R\delta LT^{-\frac{\kappa}{1+\kappa}} d^{1-\frac{1}{q}}    \Big( 1 - \frac{16d\delta^2+\log 2}{d/8} \Big)\,.
}
Note that when $d\ge 9,$ and set $\delta= \frac{1}{100},$ it holds that
\eq{
&1 - \frac{16d\delta^2+\log 2}{d/8} 
= 1- 128 \delta^2 - 8\frac{\log 2 }{d}
=  1- \frac{128}{10000} - \log 2
\ge  \frac{1}{4}\,.
%\frac{8}{8^{\frac{\kappa+1}{\kappa}}}+\frac{8\log 2}{d}\le \frac{3}{4}\,.
}
Plugging these into display~\eqref{eq:fano1} then gives
\eq{
\epsilon^*_T(\mathcal{H}_{cvx},\S,\phi) 
\ge& \frac{1}{2400} R LT^{-\frac{\kappa}{1+\kappa}} d^{1-\frac{1}{q}}   \,.
}

When $d<9,$ we restrict to the case where $d=1.$
Combining the lower bound derived in Lemma~\ref{lem:lb_cointoss4} with display~\eqref{eq:fano0} gives
\eq{
%\label{eq:lecam}
\epsilon^*_T(\mathcal{H}_{cvx},\S,\phi) 
\ge&  \frac{1}{6} R\delta LT^{-\frac{\kappa}{1+\kappa}}   \big(  1-\sqrt{8\delta^2}  \big) \,.
}
When $\delta=\frac{1}{100},$ it holds that  
\eq{
\epsilon^*_T(\mathcal{H}_{cvx},\S,\phi) 
\ge& \frac{1}{1200} RL\Big(\frac{1}{T}\Big)^{\frac{\kappa}{1+\kappa}}\,.
}
This completes the proof for the special case $\S=S_\infty(R).$
 Note that the Lipschitz constant of $g_\alpha$ does not depend on $\S$, $x^*_\alpha=\argmin_{x\in S} g_\alpha (x)\in\S,$ and thus the preceding proof goes through when $\S\supseteq  S_{\infty}(R).$
 Hence, the desired general claim follows.
 
\end{proof}

\end{document}